%% file: arxiv_v2.tex
\documentclass[11pt]{article}
\usepackage[left=1in, top=1in, bottom=1in, right=1in]{geometry}
\usepackage{times}

\usepackage{authblk}

\usepackage{mathrsfs}

\usepackage{wrapfig}
\usepackage{graphicx}
\usepackage[normalem]{ulem}
\usepackage{amsmath,amsfonts}
\usepackage{algorithm}
\usepackage{wrapfig}
\usepackage{multirow}
\usepackage{dsfont}
\usepackage{natbib}
\usepackage{comment}
\usepackage{bbm}

\usepackage{booktabs,colortbl,multirow}
\usepackage{subcaption}

\usepackage[breaklinks=true,bookmarks=true]{hyperref}

\usepackage{enumitem}

\usepackage{microtype}
\usepackage{graphicx}
\usepackage{subcaption}
\usepackage{booktabs} %
\usepackage{enumitem}
\usepackage{soul}
\usepackage{multirow}
\usepackage{array}
\usepackage{dsfont}

\usepackage[table]{xcolor}

\usepackage{hyperref}

\usepackage{amsmath}
\usepackage{amssymb}
\usepackage{mathtools}
\usepackage{amsthm}

\usepackage[capitalize]{cleveref}
	
\usepackage{color, colortbl}
\definecolor{Gray}{gray}{0.9}

\usepackage[textsize=tiny]{todonotes}

\input{math_commands.tex}

\title{\vspace{-2mm}EmbedDistill: A Geometric Knowledge Distillation\\
for Information Retrieval}

\author{Seungyeon Kim, Ankit Singh Rawat, Manzil Zaheer,\\Sadeep Jayasumana, Veeranjaneyulu Sadhanala, Wittawat Jitkrittum, \\Aditya Krishna Menon, Rob Fergus, Sanjiv Kumar}
\affil{Google LLC, USA \\
{\small \texttt{\{seungyeonk,ankitsrawat,manzilzaheer\}@google.com} \\
 \texttt{\{sadeep,veerus,wittawat,adityakmenon,robfergus,sanjivk\}@google.com}\vspace{-5mm}}
}
\date{}

\allowdisplaybreaks

\begin{document}

\maketitle

\begin{abstract}
Large neural models (such as Transformers) achieve state-of-the-art performance for information retrieval (IR). In this paper, we aim to improve distillation methods that pave the way for the resource-efficient deployment of such models in practice. Inspired by our theoretical analysis of the teacher-student generalization gap for IR models, we propose a novel distillation approach that leverages the relative geometry among queries and documents learned by the large teacher model. Unlike existing teacher score-based distillation methods, our proposed approach employs embedding matching tasks to provide a stronger signal to align the representations of the teacher and student models. In addition, it utilizes query generation to explore the data manifold to reduce the discrepancies between the student and the teacher where training data is sparse. Furthermore, our analysis also motivates novel asymmetric architectures for student models which realizes better embedding alignment without increasing online inference cost. On standard benchmarks like MSMARCO, we show that our approach successfully distills from both dual-encoder (DE) and cross-encoder (CE) teacher models to 1/10th size asymmetric students that can retain 95-97\% of the teacher performance.
\end{abstract}

\input{009_first_page}

\input{010_intro}

\input{030_background}

\input{150_theory}

\input{040_embedding_distillation}

\input{050_augmentation}

\input{060_experiments}

\input{020_related_work}

\input{070_conclusion}

%%%%%%%%%%%%%%%%%%%%%%%%%%%%%%%%%%%%%%%%%%%%%%%%%%%%%%%%%%%%%%%%%%%%%%%%%%%%%%%
%%%%%%%%%%%%%%%%%%%%%%%%%%%%%%%%%%%%%%%%%%%%%%%%%%%%%%%%%%%%%%%%%%%%%%%%%%%%%%%
% References
%%%%%%%%%%%%%%%%%%%%%%%%%%%%%%%%%%%%%%%%%%%%%%%%%%%%%%%%%%%%%%%%%%%%%%%%%%%%%%%
%%%%%%%%%%%%%%%%%%%%%%%%%%%%%%%%%%%%%%%%%%%%%%%%%%%%%%%%%%%%%%%%%%%%%%%%%%%%%%%
\bibliography{custom}
\bibliographystyle{plainnat}

%%%%%%%%%%%%%%%%%%%%%%%%%%%%%%%%%%%%%%%%%%%%%%%%%%%%%%%%%%%%%%%%%%%%%%%%%%%%%%%
%%%%%%%%%%%%%%%%%%%%%%%%%%%%%%%%%%%%%%%%%%%%%%%%%%%%%%%%%%%%%%%%%%%%%%%%%%%%%%%
% APPENDIX
%%%%%%%%%%%%%%%%%%%%%%%%%%%%%%%%%%%%%%%%%%%%%%%%%%%%%%%%%%%%%%%%%%%%%%%%%%%%%%%
%%%%%%%%%%%%%%%%%%%%%%%%%%%%%%%%%%%%%%%%%%%%%%%%%%%%%%%%%%%%%%%%%%%%%%%%%%%%%%%

\newpage
\appendix

\input{110_appendix}

\end{document}

%% file: math_commands.tex
\usepackage{amsmath,amsfonts,bm,amsthm}
\usepackage{soul}
\theoremstyle{plain}
\newtheorem{theorem}{Theorem}[section]
\newtheorem{proposition}[theorem]{Proposition}
\newtheorem{lemma}[theorem]{Lemma}

\theoremstyle{definition}

\theoremstyle{remark}

\def\eqref#1{\ref{#1}}
\def\1{\bm{1}}

\def\eps{{\epsilon}}

\newcommand{\rad}{\mathrm{Rad}}
\newcommand{\DCal}{\mathscr{D}}
\newcommand{\ECal}{\mathscr{E}}
\newcommand{\FCal}{\mathscr{F}}
\newcommand{\GCal}{\mathscr{G}}
\newcommand{\HCal}{\mathscr{H}}

\newcommand{\QCal}{\mathscr{Q}}

\newcommand{\SCal}{\mathscr{S}}

\newcommand{\UCal}{\mathscr{U}}

\newcommand{\R}{\mathbb{R}}
\DeclareMathAlphabet{\mathsfit}{\encodingdefault}{\sfdefault}{m}{sl}
\SetMathAlphabet{\mathsfit}{bold}{\encodingdefault}{\sfdefault}{bx}{n}

\newcommand{\best}[1]{\textbf{#1}}

\newcommand{\newtext}[1]{{#1}}

\newcommand{\embd}{\textsf{\small EmbedDistill}}

%% file: 009_first_page.tex
\section{Introduction}
\label{sec:intro}

Neural models for information retrieval (IR) are increasingly used to 
model the true ranking function in various applications, including web search~\citep{Mitra:2018Now}, recommendation~\citep{Zhang:2019Survey}, and question-answering (QA)~\citep{Chen:2017Reading}. Notably, the recent success of Transformers~\citep{Vaswani:2017}-based pre-trained language models~\citep{Devlin:2019, Liu:2019Roberta, Raffel:2020T5} on a wide range of natural language understanding %(NLU)
tasks has also prompted their utilization in IR to capture query-document relevance~\citep[see, e.g.,][]{Dai:2019DeepCT, MacAvaney:CEDR, Nogueira:2019, Lee:2019, Karpukhin:2020}.

A typical IR system comprises two stages: 
(1) A \emph{retriever} 
first selects a small subset of potentially relevant candidate documents (out of a large collection) for a given query; and (2) A \emph{re-ranker} then identifies a precise ranking among
the candidates provided by the retriever. \emph{Dual-encoder} (DE) models 
are the de-facto architecture for retrievers~\citep{Lee:2019, Karpukhin:2020}. Such models independently embed queries and documents into a common space, and capture their relevance by simple operations on these embeddings such as the inner product. This enables offline creation of a document index and supports fast retrieval during inference via efficient maximum inner product search implementations~\citep{Guo:2020ScaNN, Johnson:2021FAISS}\newtext{, with \emph{online} query embedding generation primarily dictating the inference latency.} \emph{Cross-encoder} (CE) models, on the other hand, are preferred as re-rankers, owing to their excellent performance~\citep{Nogueira:2019, Dai:2019, Yilmaz:2019}. A CE model jointly encodes a query-document pair while enabling early interaction among query and document features. Employing a CE model for retrieval is often infeasible, as it would require processing a given query with \emph{every} document in the collection at inference time. In fact, even in the re-ranking stage, the inference cost of CE models is high enough~\citep{Khattab:2020} to warrant exploration of efficient alternatives~\citep{Hofstatter:2020, Khattab:2020, Menon:2022DE}. {Across both architectures, scaling to larger models brings improved performance at increased computational cost~\citep{ni-etal-2022-large,neelakantan2022text}.}

%% file: 010_intro.tex
\emph{Knowledge distillation}~\citep{Bucilua:2006,Hinton:2015} provides a general strategy to address the prohibitive inference cost associated with high-quality large neural models. In the IR literature, most existing distillation methods only rely on the teacher's query-document relevance scores~\citep[see, e.g.,][]{Lu:2020TwinBERT, Hofstatter:2020, Chen:2021KD, Ren:2021RocketV2, Santhanam:2021} 
or their proxies~\citep{Izacard:2021Distilling}. 
However, given that neural IR models are inherently embedding-based, it is natural to ask:
\begin{center}
\emph{Is it useful to go beyond matching of the teacher and student models' \emph{scores},\\and directly aim to align their \emph{embedding spaces}?}
\end{center}

With this in mind, we propose a novel distillation method for IR models that utilizes an \emph{embedding matching} task to train 
student models. The proposed method is inspired by our rigorous treatment of the generalization gap between the teacher and student models in IR settings. Our theoretical analysis of the \textit{teacher-student generalization gap} further suggests novel design choices involving \textit{asymmetric configurations} for student DE models, intending to further reduce the gap by better aligning teacher and student embedding spaces. Notably, our proposed distillation method supports \emph{cross-architecture distillation} and improves upon existing (score-based) distillation methods for both retriever and re-ranker models. When distilling a large teacher DE model into a smaller student DE model, for a given query (document), one can minimize the distance between the query (document) embeddings of the teacher and student (after compatible projection layers to account for dimension mismatch, if any). In contrast, a teacher CE model doesn't directly provide document and query embeddings, and so to effectively employ embedding matching-based distillation  requires modifying the scoring layer with \emph{dual-pooling}~\citep{yadav2022efficient} and adding various regularizers.
Both of these changes improve geometry of teacher embeddings and facilitate effective knowledge transfer to the student DE model via embedding matching-based distillation.

Our key contributions toward improving IR models via distillation are:
\begin{list}{\textbullet}{\leftmargin=1.35em \itemindent=0em \itemsep=0pt\parsep=1mm}
\vspace{-3mm}
    \item We provide the first rigorous analysis of the teacher-student generalization gap for IR settings which captures the role of alignment of embedding spaces of the teacher and student towards reducing the gap (Sec.~\ref{sec:theory}).
    \item Inspired by our analysis, we propose a novel distillation approach for neural IR models, namely \embd, that goes beyond score matching and aligns the embedding spaces of the teacher and student models (Sec.~\ref{sec:embed_distill}). We also show that \embd~can leverage synthetic data to improve a student by further aligning the embedding spaces of the teacher and student (Sec.~\ref{sec:augmentation}).
    \item Our analysis motivates novel distillation setups. Specifically, we consider a student DE model with an \emph{asymmetric} configuration, consisting of a small query encoder and a \textit{frozen} document encoder inherited from the teacher. This 
    significantly reduces inference latency of query embedding generation, while leveraging the teachers' high-quality document index (Sec.~\ref{sec:detode}). 
    \item We provide a \textit{comprehensive} empirical evaluation of \embd~(Sec.~\ref{sec:experiments}) on two standard IR benchmarks -- Natural Questions~\citep{Kwiatkowski:2019NQ} and MSMARCO~\citep{Nguyen:2016}. We also evaluate \embd~on BEIR benchmark~\citep{thakur2021beir} which is used to measure the \textit{zero-shot} performance of an IR model.
\vspace{-2mm}
\end{list}

\begin{figure}[t]
    \vspace{-4mm}
    \centering
    \begin{subfigure}[b]{0.7\textwidth}
        \centering
        \includegraphics[page=2,width=\linewidth,trim=0.5cm 6cm 6cm 0.5cm,clip]{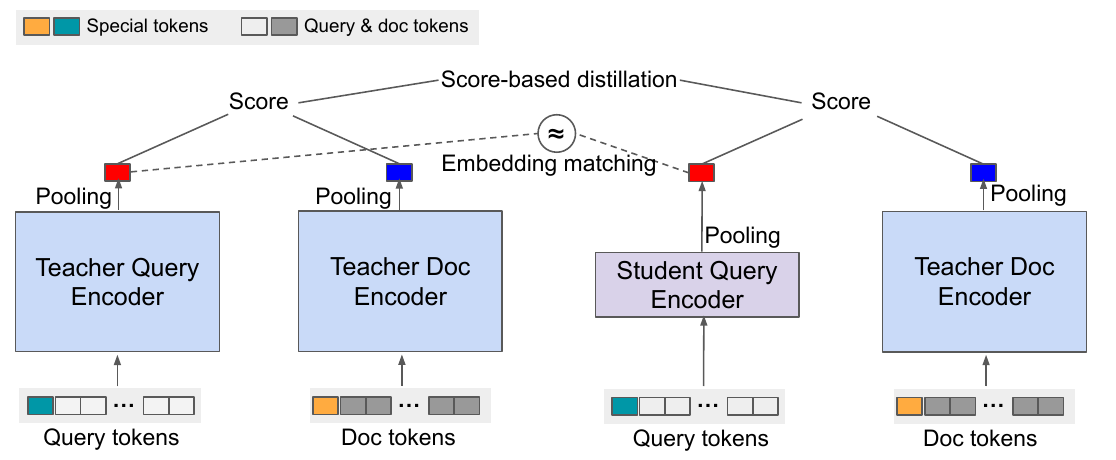}
        \caption{DE to DE distillation \vspace*{3mm}}
        \label{fig:de2de-embed}
    \end{subfigure}
    \begin{subfigure}[b]{0.7\textwidth}
        \centering
        \includegraphics[page=4,width=\linewidth,trim=0.5cm 6.3cm 8cm 0.5cm,clip]{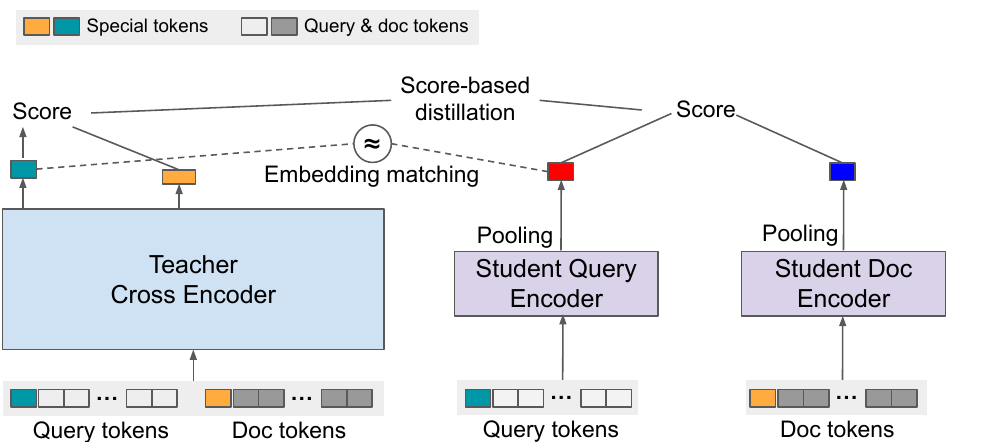}
        \caption{CE to DE distillation}
        \label{fig:ce2de-embed}
    \end{subfigure}
    \caption{Proposed distillation method with query embedding matching. 
    \textbf{Top:} In the DE to DE distillation setting, the student employs an asymmetric DE configuration with a small query encoder and a large (non-trainable) document encoder inherited from the teacher DE model. {The smaller query encoder ensures small latency for encoding query during inference, and large document encoder leads to a good quality document index. \textbf{Bottom:} Similarly the setting of CE to DE distillation using \embd, with teacher CE model employing dual pooling.}
    }
    \label{fig:distill-embed}
    \vspace{-3mm}
\end{figure}

Note that prior works have utilized embedding alignment during distillation for {\em non-IR} setting~\citep[see, e.g.,][]{romero2014fitnets,Sanh:2019DistilBERT,Jiao:2020TinyBERT,aguilar2020knowledge,zhang2020improve, chen2022knowledge}. However, to the best of our knowledge, our work is the first to study embedding matching-based distillation method for IR settings which requires addressing multiple IR-specific challenges such as cross-architecture distillation, partial representation alignment, and enabling novel asymmetric student configurations. Furthermore, unlike these prior works, our proposed method is theoretically justified to reduce the teacher-student performance gap.

%% file: 030_background.tex
\section{Background}
\label{sec:bg}

Let $\QCal$ and $\DCal$ denote the query and document spaces, respectively. An IR model is equivalent to a scorer $s: \QCal \times \DCal \to \R$, i.e., it assigns a (relevance) score $s(q, d)$ for a query-document pair $(q, d) \in \QCal \times \DCal$. Ideally, we want to learn %an IR model or 
a scorer such that $s(q, d) > s(q, d')$ \emph{iff} the document $d$ is more relevant to the query $q$ than document $d'$. We assume access to $n$ labeled training examples $\SCal_n = \{(q_i, \mathbf{d}_i, \mathbf{y}_i)\}_{i \in [n]}$. 
Here, $\mathbf{d}_{i} = (d_{i,1},\ldots, d_{i,L}) \in \DCal^L,~\forall i \in [n]$, denotes a list of $L$ documents and $\mathbf{y}_{i} = (y_{i,1},\ldots, y_{i,L}) \in \{0,1\}^L$ denotes the corresponding labels such that $y_{i,j} = 1$ iff the document $d_{i,j}$ is relevant to the query $q_i$. Given $\SCal_n$, we learn an IR model by minimizing 
\vspace{-1mm}
\begin{equation}
\label{eq:erm-one-hot}
R(s; \SCal_n) := \frac{1}{n}\sum\nolimits_{i \in [n]} \ell\big( s_{q_i, \mathbf{d}_i}, \mathbf{y}_i\big),
\end{equation}
where $s_{q_i, \mathbf{d}_i} := (s(q_i, d_{1,i}),\ldots, s(q_i, d_{1,L}))$ and  $\ell\big( s_{q_i, \mathbf{d}_i}, \mathbf{y}_i\big)$ denotes the loss 
$s$ incurs on $(q_i, \mathbf{d}_i, \mathbf{y}_i)$. Due to space constraint, we defer concrete choices for the loss function $\ell$ to Appendix~\ref{appendix:loss}.

While this learning framework is general enough to work with any IR models, next, we formally introduce two families of Transformer-based IR models that are prevalent in the recent literature.

\vspace{-1mm}
\subsection{Transformer-based IR models: Cross-encoders and Dual-encoders}
\vspace{-1mm}

Let query $q$ = $(q^1,\ldots, q^{m_1})$ and document $d$ = $(d^1,\ldots,  d^{m_2})$ consist of $m_1$ and $m_2$ tokens, respectively. We now discuss how Transformers-based CE and DE models process the $(q, d)$ pair. 

\vspace{-0.5mm}
\noindent\textbf{Cross-encoder model.}~Let $p = [q; d]$ be the %$m_1 + m_2$ length 
sequence obtained by concatenating $q$ and $d$. Further, let $\tilde{p}$ be the sequence obtained by adding special tokens such {\tt[CLS]} and {\tt[SEP]} to $p$. Given an encoder-only Transformer model ${\rm Enc}$, the relevance score for the $(q, d)$ pair is 
\vspace{-1mm}
\begin{equation}
\label{eq:ce-desc}
    s(q, d) = \langle w, {\rm pool}\big({\rm Enc}(\tilde{p})\big) \rangle = \langle w, {\tt emb}_{q, d} \rangle,
\end{equation}
where $w$ is a $d$-dimensional classification vector, and ${\rm pool}(\cdot)$ denotes a pooling operation that transforms the contextualized token embeddings ${\rm Enc}(\tilde{p})$ to a joint embedding vector ${\tt emb}_{q, d}$. {\tt[CLS]}-pooling is a common operation that simply outputs the embedding of the {\tt[CLS]} token as ${\tt emb}_{q, d}$.

\vspace{-0.5mm}
\noindent\textbf{Dual-encoder model.}~Let $\tilde{q}$ and $\tilde{d}$ be the sequences obtained by adding appropriate special tokens to $q$ and $d$, respectively. A DE model comprises two (encoder-only) Transformers ${\rm Enc}_Q$ and ${\rm Enc}_D$, which we call query and document encoders, respectively.\footnote{It is common to employ dual-encoder models where query and document encoders are shared.} Let ${\tt emb}_q$ = ${\rm pool}\big({\rm Enc}_{Q}(\tilde{q})\big)$ and ${\tt emb}_d$ = ${\rm pool}\big({\rm Enc}_{D}(\tilde{d})\big)$ denote the query and document embeddings, respectively. Now, one can define 
$s(q,d) = \langle {\tt emb}_q,  {\tt emb}_d\rangle$
to be the relevance score assigned to the $(q,d)$ pair by the DE model.

\vspace{-1mm}
\subsection{Score-based distillation for IR models}
\label{sec:ir-distill}
\vspace{-1mm}

Most distillation schemes for IR~\citep[e.g.,][]{Lu:2020TwinBERT, Hofstatter:2020, Chen:2021KD} 
rely on teacher relevance scores. Given a training set $\SCal_n$ %$=  \{(q_i, \mathbf{d}_i, \mathbf{y}_i)\}_{i \in [n]}$ 
and a teacher with scorer $s^{\rm t}$, one learns a student with scorer $s^{\rm s}$ by minimizing
\vspace{-1mm}
\begin{equation}
\label{eq:emp_disk_distill}
R(s^{\rm s}, s^{\rm t}; \SCal_n) = \frac{1}{n} \sum\nolimits_{i \in [n]} \ell_{\rm d}\big(s^{\rm s}_{q, \mathbf{d}_i}, s^{\rm t}_{q, \mathbf{d}_i} \big),
\vspace{-1mm}
\end{equation}
where $\ell_{\rm d}$ captures the discrepancy between $s^{\rm s}$ and $s^{\rm t}$. See Appendix~\ref{appendix:loss} for common choices for $\ell_{\rm d}$.

%% file: 150_theory.tex
\section{Teacher-student generalization gap: Inspiration for embedding alignment}
\label{sec:theory}

Our main objective is to devise novel distillation methods to realize high-performing student DE models. As a first step in this direction, we rigorously study the teacher-student generalization gap as realized by standard (score-based) distillation in IR settings. Informed by our analysis, we subsequently identify novel ways to improve the student model's performance. In particular, our analysis suggests two natural directions to reduce the teacher-student generalization gap: 1) enforcing tighter alignment between embedding spaces of teacher and student models; and 2) exploring novel asymmetric configuration for student DE model. 

{Let $R(s) = \mathbb{E}\left[\ell\big( s_{q, \mathbf{d}}, \mathbf{y}\big)\right]$ be the population version of the empirical risk in Eq.~\ref{eq:erm-one-hot}, which measures the test time performance of the IR model defined by the scorer~$s$. 
Thus, $R(s^{\rm s}) - R(s^{\rm t})$ denotes the \textit{teacher-student generalization gap}. In the following result, we bound this quantity (see Appendix~\ref{sec:bin_ce_risk} for a formal statement and proof). We focus on distilling a teacher DE model to a student DE model and $L=1$ (cf.~Sec.~\ref{sec:bg}) as it leads to easier exposition without changing the main takeaways. Our analysis can be extended to $L > 1$ or CE to DE distillation with more complex notation.}

\begin{theorem}[Teacher-student generalization gap (informal)]
\label{thm:teacher-student-gap}
Let $\FCal$ and $\GCal$ denote the function classes for the query and document encoders for the student model, respectively. Suppose that the score-based distillation loss $\ell_{\rm d}$ in Eq.~\eqref{eq:emp_disk_distill} is based on binary cross entropy loss (Eq.~\eqref{eq:binary-ce-distill} in Appendix~\ref{appendix:loss}). Let one-hot (label-dependent) loss $\ell$ in Eq.~\eqref{eq:erm-one-hot} %$\ell^{y}(q,d)$ 
be the binary cross entropy loss (Eq.~\eqref{eq:binary-ce-one-hot} in Appendix~\ref{appendix:loss}). Further, assume that all encoders have the same output dimension and embeddings have their $\ell_2$-norm bounded by $K$. Then, we have
\begin{equation}
\begin{aligned}
&R(s^{\rm s}) - R(s^{\rm t}) \leq \ECal_n (\FCal, \GCal) + 2K R_{{\rm Emb}, Q}({\rm t}, {\rm s}; \SCal_n) + 2K R_{{\rm Emb}, D}({\rm t}, {\rm s}; \SCal_n) \\
&\qquad \qquad \qquad \quad + \Delta(s^{\rm t}; \SCal_n) + K^2\big(\mathbb{E}\left[\left|\sigma(s^{\rm t}_{q,d})-y\right|\right] + \frac{1}{n}\sum_{i \in [n]}\left|\sigma(s^{\rm t}_{q_i,d_i})-y_i\right|\big),
\end{aligned}
\label{eq:ts-gap}
\end{equation}
where  
$\ECal_n (\FCal, \GCal) := 
\sup_{s^{\rm s} \in \FCal \times \GCal} \big|R(s^{\rm s}, s^{\rm t}; \SCal_n) - \mathbb{E}\ell_{\rm d} \big( s^{\rm s}_{q,d}, s_{q,d}^{\rm t}\big)\big| %\nonumber
$; 
$\sigma$ denotes the sigmoid function; and $\Delta(s^{\rm t}; \SCal_n)$ denotes the deviation between the empirical risk (on $\SCal_n$) and population risk of the teacher $s^{\rm t}$. {Here, $R_{{\rm Emb}, Q}({\rm t}, {\rm s}; \SCal_n)$ and $R_{{\rm Emb}, D}({\rm t}, {\rm s}; \SCal_n)$ measure misalignment between teacher and student embeddings by focusing on queries and documents, respectively (cf.~Eq.~\ref{eq:emb-loss-de2de-q} \& \ref{eq:emb-loss-de2de-d} in Sec.~\ref{sec:detode}).}
\end{theorem}

The last three quantities in the bound in Thm.~\ref{thm:teacher-student-gap}, namely $\Delta(s^{\rm t}; \SCal_n)$, $\mathbb{E}[|\sigma(s^{\rm t}_{q,d})-y|]$, and $\frac{1}{n}\sum_{i \in [n]}|\sigma(s^{\rm t}_{q_i,d_i})-y_i|$, are \textit{independent} of the underlying student model. These terms solely depend on the quality of the underlying teacher model $s^{\rm t}$. That said, the teacher-student gap can be made small by reducing the following three terms: 1) uniform deviation of the student's empirical distillation risk from its population version $\ECal_n (\FCal, \GCal)$; 2) misalignment between teacher student query embeddings $R_{{\rm Emb}, Q}({\rm t}, {\rm s}; \SCal_n)$; and 3) misalignment between teacher student document embeddings $R_{{\rm Emb}, D}({\rm t}, {\rm s}; \SCal_n)$. 

The last two terms motivate us to propose an \textit{embedding matching}-based distillation that explicitly aims to minimize these terms during student training. Even more interestingly, these terms also inspire an \textit{asymmetric DE configuration} for the student which strikes a balance between the goals of reducing the misalignment between the embeddings of teacher and student (by inheriting teacher's document encoder) and ensuring {serving efficiency (small inference latency)} by employing a small query encoder. Before discussing these proposals in detail in Sec.~\ref{sec:embed_distill} and Fig.~\ref{fig:distill-embed}, we explore the first term $\ECal_n (\FCal, \GCal)$ and highlight how our proposals also have implications for reducing this term. Towards this, the following result bounds $\ECal_n (\FCal, \GCal)$. We present an informal statement of the result (see Appendix~\ref{appen:unif_conv_bd} for a more precise statement and proof).

\begin{proposition}
\label{prop:lower_complexity_with_frozen_docencoder-main} Let $\ell_{\rm d}$ be a distillation loss which is $L_{\ell_{\rm d}}$-Lipschitz in its first argument. Let $\FCal$ and $\GCal$ denote the function classes for the query and document encoders, respectively. Further assume that, for each query and document encoder in our function class, the query and document embeddings have their $\ell_2$-norm bounded by $K$. Then,
\begin{equation}
  \label{eq:deviation-1}
  \ECal_n (\FCal, \GCal) 
   \leq
  \mathbb{E}_{\SCal_n}\frac{48 K L_{\ell_{\rm d}}}{\sqrt{n}} \int_{0}^\infty \sqrt{\log \big(N(u, \FCal)N(u, \GCal)\big)} \;du.
\end{equation}
Furthermore, with a fixed document encoder, i.e., $\GCal = \{g^*\}$, % we have
\begin{equation}
 \ECal_n (\FCal, \{g*\}) \leq \mathbb{E}_{\SCal_n} \frac{48 K L_{\ell_{\rm d}}}{\sqrt{n}} \int_{0}^\infty \sqrt{\log N(u, \FCal)} \;du.
  \label{eq:deviation-2}
\end{equation}
Here, $N(u, \cdot)$ is the $u$-covering number of a function class.
\end{proposition}

Note that Eq.~\eqref{eq:deviation-1} and Eq.~\ref{eq:deviation-2} correspond to uniform deviation when we train \emph{without} and \emph{with} a frozen document encoder, respectively.
It is clear that the bound in Eq.~\ref{eq:deviation-2} is less than or equal to that in Eq.~\eqref{eq:deviation-1} (because $N(u, \GCal) \ge 1$ for any $u$),
which alludes to desirable impact of employing a frozen document encoder as one of our proposal seeks to do via \textit{inheriting teacher's document encoder} (for instance in an asymmetric DE configuration). Furthermore, our proposal of employing an embedding-matching task will regularize the function class of query encoders; effectively reducing it to $\FCal'$ with $|\FCal'| \leq |\FCal|$. The same holds true for document encoder function class when document encoder is trainable (as in Eq.~\eqref{eq:deviation-1}), leading to an effective function class $\GCal'$ with $|\GCal'| \leq |\GCal|$. Since we would have $N(u,\FCal') \leq N(u, \FCal)$ and $N(u,\GCal') \leq N(u, \GCal)$, this suggests desirable implications of embedding matching for reducing the uniform deviation bound.

%% file: 040_embedding_distillation.tex
\section{Embedding-matching based distillation}
\label{sec:embed_distill}

Informed by our analysis of teacher-student generalization gap in Sec.~\ref{sec:theory}, we propose \embd~-- a novel distillation method that explicitly focuses on aligning the embedding spaces of the teacher and student. Our proposal goes beyond existing distillation methods in the IR literature that only use the teacher scores. Next, we introduce \embd~for two prevalent settings: (1) distilling a large DE model to a smaller DE model;
\footnote{
CE to CE distillation is a special case of this with classification vector $w$ (cf.~Eq.~\eqref{eq:ce-desc}) as trivial second encoder.} 
and (2) distilling a CE model to a DE model.

\vspace{-1mm}
\subsection{DE to DE distillation}
\label{sec:detode}
\vspace{-1mm}
Given a $(q, d)$ pair, let ${\tt emb}^{\rm t}_q$ and ${\tt emb}^{\rm t}_d$ be the query and document embeddings produced by the query encoder ${\rm Enc}^{\rm t}_{Q}$ and document encoder ${\rm Enc}^{\rm t}_{D}$ of the teacher DE model, respectively. Similarly, let ${\tt emb}^{\rm s}_q$ and ${\tt emb}^{\rm s}_d$ denote the query and document embeddings produced by a student DE model with $({\rm Enc}^{\rm s}_{Q}, {\rm Enc}^{\rm s}_{D})$ as its query and document encoders. Now, \embd~optimizes the following embedding alignment losses in addition to the score-matching loss from Sec.~\ref{sec:ir-distill} to align query and document embeddings of the teacher and student:
{
\begin{align}
\label{eq:emb-loss-de2de-q}
R_{{\rm Emb}, Q}({\rm t}, {\rm s}; \SCal_n) 
= \frac{1}{n}\sum\nolimits_{q \in \SCal_n} \|{\tt emb}^{\rm t}_{q}  - {\rm proj}\big({\tt emb}^{\rm s}_{q}\big) \|; \\
R_{{\rm Emb}, D}({\rm t}, {\rm s}; \SCal_n) 
% \nonumber \\ &\qquad
= \frac{1}{n}\sum\nolimits_{d \in \SCal_n} \|{\tt emb}^{\rm t}_{d}  - {\rm proj}\big({\tt emb}^{\rm s}_{d}\big) \|. \label{eq:emb-loss-de2de-d}
\end{align}
}%

\noindent\textbf{Asymmetric DE.}~We also propose a novel student DE configuration where the student employs the teacher's document encoder (i.e., ${\rm Enc}^{\rm s}_{D} = {\rm Enc}^{\rm t}_{D}$) and only train its query encoder, which is much smaller compared to the teacher's query encoder.
For such a setting, it is natural to only employ the embedding matching loss in Eq.~\eqref{eq:emb-loss-de2de-q} as the document embeddings are aligned by design (cf.~Fig.~\ref{fig:de2de-embed}).

Note that this asymmetric student DE does not incur an increase in latency despite the use of a large teacher document encoder. This is because the large document encoder is only needed to create a good quality document index offline, and only the query encoder is evaluated at inference time. Also, the similarity search cost is not increased as the projection layer ensures the same small embedding dimension as in the symmetric DE student.
Thus, for DE to DE distillation, we prescribe the asymmetric DE configuration universally. Our theoretical analysis (cf.~Sec.~\ref{sec:theory}) and experimental results (cf.~Sec.~\ref{sec:experiments}) suggest that the ability to inherit the document tower from the teacher DE model can drastically improve the final performance, especially when combined with query embedding matching task (cf.~Eq.~\ref{eq:emb-loss-de2de-q}).

\vspace{-1mm}
\subsection{CE to DE distillation}
\label{sec:cetode}
\vspace{-1mm}
Given that CE models jointly encode query-document pairs, individual query and document embeddings are not readily available to implement embedding matching losses as per Eq.~\ref{eq:emb-loss-de2de-q} and \ref{eq:emb-loss-de2de-d}. This makes it challenging to employ \embd~for CE to DE distillation.

As a na\"{i}ve solution, for a \smash{$(q, d)$} pair, one can simply match a joint transformation of the student's query embedding \smash{${\tt emb}^{\rm s}_q$} and document embedding ${\tt emb}^{\rm s}_d$ to the teacher's joint embedding ${\tt emb}^{\rm t}_{q, d}$ , produced by (single) teacher encoder \smash{${\rm Enc}^{t}$}. However, {we observed} that including such an embedding matching task often leads to severe over-fitting, and results in a poor student. 
Since
$
s^{\rm t}(q, d) = \langle w, {\tt emb}^{\rm t}_{q, d}\rangle
$, during CE model training, the joint embeddings ${\tt emb}^{\rm t}_{q, d}$ for relevant and irrelevant $(q, d)$ pairs are encouraged to be aligned with $w$ and $-w$, respectively. {This produces degenerate embeddings that do not capture semantic query-to-document relationships. We notice that even the final query and document token embeddings lose such semantic structure (cf.~Appendix~\ref{appen:embed-analysis-ce2de}). Thus, a teacher CE model with $s^{\rm t}(q, d) = \langle w, {\tt emb}^{\rm t}_{q, d}\rangle$ 
does not add value for distillation beyond score-matching;
in fact, it \emph{hurts} to include na\"{i}ve embedding matching.}
Next, we propose a modified CE model training strategy that facilitates \embd.

\noindent \textbf{CE models with dual pooling.}
A \emph{dual pooling} scheme is employed in the scoring layer to produce two embeddings ${\tt emb}^{\rm t}_{q \leftarrow (q,d)}$ and ${\tt emb}^{\rm t}_{d \leftarrow (q,d)}$ from a CE model that serve as the \emph{proxy} query and document embeddings, respectively. Accordingly, we define the relevance score as $s^{\rm t}(q, d) = \langle {\tt emb}^{\rm t}_{q \leftarrow (q,d)},{\tt emb}^{\rm t}_{d \leftarrow (q,d)}\rangle$. We explore two variants of dual pooling: {(1) special token-based pooling that pools from \texttt{[CLS]} and \texttt{[SEP]}; and (2) segment-based weighted mean pooling that separately performs weighted averaging on the query and document segments of the final token embeddings. See Appendix~\ref{appendix:pooling} for details.}

In addition to dual pooling, we also utilize a reconstruction loss during the CE training, which measures the likelihood of predicting each token of the original input from the final token embeddings. This loss encourages reconstruction of query and document tokens based on the final token embeddings and prevents the degeneration of the token embeddings during training. Given proxy embeddings from the teacher CE, we can perform \embd~with the embedding matching loss defined in Eq.~\eqref{eq:emb-loss-de2de-q} and Eq.~\eqref{eq:emb-loss-de2de-d} (cf.~Fig.~\ref{fig:ce2de-embed}).

%% file: 050_augmentation.tex
\subsection{{Task-specific online data generation}}
\label{sec:augmentation}

Data augmentation as a general technique has been previously considered in the IR literature~\citep[see, e.g., ][]{Nogueira:2019docTOquery, ouguz2021domain, Izacard:2021Contriever}, especially in data-limited, out-of-domain, or zero-shot settings.
As \embd~aims to align the embeddings spaces of the teacher and student, the ability to generate similar queries or documents can naturally help enforce such an alignment globally on the task-specific manifold. Given a set of unlabeled task-specific query and document pairs $\UCal_m$,
we can further add the embedding matching losses $R_{\rm Emb, Q}({\rm t}, {\rm s}; \UCal_m)$ or $R_{\rm Emb, D}({\rm t}, {\rm s}; \UCal_m)$ to our training objective.
Interestingly, for DE to DE distillation setting, our approach can even benefit from a large collection of task-specific queries $\QCal'$ or documents $\DCal'$. Here, we can independently employ embedding matching losses $R_{\rm Emb, Q}({\rm t}, {\rm s}; \QCal')$ or $R_{\rm Emb, D}({\rm t}, {\rm s}; \DCal')$ that focus on queries and documents, respectively. Please refer to Appendix~\ref{appendix:query-generation} describing how the task-specific data were generated.

%% file: 060_experiments.tex
\section{Experiments}
\label{sec:experiments}

\begin{table}[t]
\caption{\textit{Full} recall performance of various student DE models on NQ dev set, including symmetric DE student model (67.5M or 11.3M transformer for both encoders), and asymmetric DE student model (67.5M or 11.3M transformer as query encoder and document embeddings inherited from the teacher). All distilled students used the same teacher (110.1M parameter BERT-base models as both encoders), with the full  Recall@5 = 72.3, Recall@20 = 86.1, and Recall@100 = 93.6.}
    \label{tab:nq-dev}
    \centering
    \scalebox{0.85}{
    \begin{tabular}{@{}lc@{\hspace{2.0mm}}c@{ }cc@{\hspace{-2mm} }c@{\hspace{2.0mm}}c@{ }c@{}}
    \toprule
    \multirow{2}{*}{\textbf{Method}} & \multicolumn{3}{c}{\textbf{6-Layer (67.5M)}} && \multicolumn{3}{c}{\textbf{4-Layer (11.3M)}}\\
    \cmidrule(r){2-4} \cmidrule{6-8}  
    & \textbf{R@5} & \textbf{R@20} & \textbf{R@100} && \textbf{R@5} & \textbf{R@20} & \textbf{R@100} \\
    \midrule
    Train student directly             & 36.2 & 59.7 & 80.0 && 24.8 & 44.7 & 67.5 \\
    \; + Distill from teacher        & 65.3 & 81.6 & 91.2 && 44.3 & 64.9 & 81.0 \\
    \; + Inherit doc embeddings & 69.9 & 83.9 & 92.3 && 56.3 &  70.9 & 82.5 \\
    \; + Query embedding matching    & 72.7 & \best{86.5} & \best{93.9} && 61.2 & 75.2 & 85.1 \\
    \; + Query generation            & \best{73.4} & \best{86.3} & \best{93.8} && \best{64.3} & \best{77.8} & \best{87.9} \\
    \midrule
    Train student using only\\
    embedding matching and \\
    inherit doc embeddings & 71.4 & 84.9 & 92.6 &&  64.6 & 50.2  & 76.8 \\
    \; + Query generation           & 71.8 &  85.0 & 93.0  && 54.2 & 68.9 & 80.8 \\
    \bottomrule
    \end{tabular}}
\end{table}

We now conduct a comprehensive evaluation of the proposed distillation approach. Specifically, we highlight the utility of the approach for both DE to DE and CE to DE distillation. We also showcase the benefits of combining our distillation approach with query generation methods.

\subsection{Setup}

\noindent\textbf{Benchmarks and evaluation metrics.}~We consider two popular IR benchmarks --- Natural Questions (NQ)~\citep{kwiatkowski2019natural} and MSMARCO~\citep{Nguyen:2016}, which focus on finding the most relevant passage/document given a question and a search query, respectively. NQ provides both standard test and dev sets, whereas MSMARCO {provides only the dev set that are widely used for common benchmarks.} In what follows, we use the terms query (document) and question (passages) interchangeably. For NQ, {we use the standard full recall (\emph{strict}) as well as the \emph{relaxed} recall metric \citep{Karpukhin:2020} to evaluate the retrieval performance. 
For MSMARCO, we focus on the standard metrics \emph{Mean Reciprocal Rank} (MRR)@10, and \emph{normalized Discounted Cumulative Gain} (nDCG)@10 to evaluate both re-ranking and retrieval performance. 
For the re-ranking, we restrict to re-ranking only the top 1000 candidate document provided as part of the dataset to be fair, while some works use stronger methods to find better top 1000 candidates for re-ranking (resulting in higher evaluation numbers)

See Appendix~\ref{appendix:evaluation} for a detailed discussion on these evaluation metrics.} Finally, we also evaluate \embd~on the BEIR benchmark~\citep{thakur2021beir} in terms of nDCG@10 and recall@100 metrics.

\noindent\textbf{Model architectures.}~We follow the standard Transformers-based IR model architectures similar to~\citet{Karpukhin:2020,Qu:2021Rocket,ouguz2021domain}. We utilized various sizes of DE models based on BERT-base~\citep{Devlin:2019} (12-layer, 768 dim, 110M parameters), DistilBERT~\citep{Sanh:2019DistilBERT} (6-layer, 768 dim, 67.5M parameters -- $\sim$ 2/3 of base), or BERT-mini~\citep{turc2019well} (4-layer, 256 dim, 11.3M parameters -- $\sim$ 1/10 of base). For query generation (cf.~Sec.~\ref{sec:augmentation}), we employ BART-base~\citep{Lewis:2020BART}, an encoder-decoder model, to generate similar questions from each training example's input question (query). We randomly mask $10\%$ of
tokens and inject zero mean Gaussian noise with $\sigma=\{0.1,0.2\}$ between the encoder and decoder. See Appendix~\ref{appendix:query-generation} for more details on query generation and Appendix~\ref{appen:additional-training-details} for hyperparameters.

\begin{table}[t]
\caption{Performance of \embd~for DE to DE distillation on NQ test set. While prior works listed in the table rely on techniques such as negative mining and multi-stage training, we explore the orthogonal direction of embedding-matching that improves {\em single-stage} distillation, which can be combined with them.}
    \label{tab:nq_testset}
    \centering
    \scalebox{0.85}{
    \renewcommand{\arraystretch}{1.0}
    \begin{tabular}{@{}lc@{\;\;}cc@{}}
    \toprule
    \textbf{Method} & \textbf{\#Layers} & \textbf{R@20} & \textbf{R@100} \\
    \midrule
    DPR~\citep{Karpukhin:2020}             & 12 & 78.4 & 85.4 \\
    DPR + PAQ~\citep{ouguz2021domain}      & 12 & 84.0 & 89.2 \\
    DPR + PAQ~\citep{ouguz2021domain}      & 24 & 84.7 & 89.2 \\
    ACNE~\citep{Xiong:2021ANCE}            & 12 & 81.9 & 87.5 \\
    RocketQA~\citep{Qu:2021Rocket}         & 12 & 82.7 & 88.5 \\
    MSS-DPR~\citep{sachan-etal-2021-end}   & 12 & 84.0 & 89.2 \\
    MSS-DPR~\citep{sachan-etal-2021-end}   & 24 & 84.8 & 89.8 \\
    \midrule
    Our teacher~\citep{zhang2022ar2} & 12 (220.2M) & 85.4 & 90.0 \\
    \embd &  6 (67.5M) & 85.1 & 89.8 \\
    \embd &  4 (11.3M) & 81.2 & 87.4 \\
    \bottomrule
    \end{tabular}}
% \vspace{-3mm}
\end{table}

\subsection{DE to DE distillation}
\label{sec:exp-de-to-de-main}

We employ AR2~\citep{zhang2022ar2}\footnote{\url{https://github.com/microsoft/AR2/tree/main/AR2}} and SentenceBERT-v5~\citep{reimers-2019-sentence-bert}\footnote{\url{https://huggingface.co/sentence-transformers/msmarco-bert-base-dot-v5}} as teacher DE models for NQ and MSMARCO. Note that both models are based on BERT-base. For DE to DE distillation, we consider two kinds of configurations for the student DE model: 
    (1) \emph{Symmetric}:~We use identical question and document encoders. We evaluate DistilBERT and BERT-mini on both datasets. 
    (2) \emph{Asymmetric}:~The student inherits document embeddings from the teacher DE model and \emph{are not} trained during the distillation. For query encoder, we use DistilBERT or BERT-mini which are smaller than document encoder. 

\noindent\textbf{Student DE model training.}
We train student DE models using a combination of (i) one-hot loss (cf.~Eq.~\eqref{eq:ce-one-hot} in Appendix~\ref{appendix:loss}) on training data; (ii) distillation loss in (cf.~Eq.~\eqref{eq:ce-distill} in Appendix~\ref{appendix:loss}); and (iii) embedding matching loss in Eq.~\eqref{eq:emb-loss-de2de-q}.
We used \texttt{[CLS]}-pooling for all student encoders. Unlike DPR~\citep{Karpukhin:2020} or AR2, we do not use hard negatives from BM25 or other models, which greatly simplifies our distillation procedure.

\noindent\textbf{Results and discussion.}
To understand the impact of various proposed configurations and losses, we train models by sequentially adding components {
and evaluate their retrieval performance on NQ and MSMARCO dev set as shown in Table~\ref{tab:nq-dev} and Table~\ref{tab:msmarco_dede} respectively. (See Table~\ref{tab:nq-dev-relaxed} in Appendix~\ref{appendix:additional-exps-nq} for performance on NQ in terms of the relaxed recall and Table~\ref{tab:msmarco_dede_ndcg} in Appendix~\ref{appendix:additional-exps-msmarco} for MSMARCO in terms of nDCG@10.)
}

We begin by training a symmetric DE without distillation. As expected, moving to distillation brings in considerable gains. Next, we swap the student document encoder with document embeddings from the teacher (non-trainable), which leads to a good jump in the performance. Now we can introduce \embd~with Eq.~\eqref{eq:emb-loss-de2de-q} for aligning query representations between student and teacher. 
{The two losses are combined with weight of $1.0$ (except for BERT-mini models in the presence of query generation with $5.0$).} This improves performance significantly, e.g.,it provides $\sim$3 and $\sim$5 points increase in recall@5 on NQ with students based on DistilBERT and BERT-mini, respectively (Table~\ref{tab:nq-dev}). We further explore the utility of \embd~in aligning the teacher and student embedding spaces in Appendix~\ref{appen:embed-analysis-de2de}.

\begin{table}[t]
    \centering
    \caption{Performance of various DE models on MSMARCO dev set for both \emph{re-ranking} and \emph{retrieval} tasks ({full corpus}). The teacher model (110.1M parameter BERT-base models as both encoders) for re-ranking achieves MRR@10 of 36.8 and that for retrieval get MRR@10 of 37.2.  The table shows performance (in MRR@10) of the symmetric DE student model (67.5M or 11.3M transformer as both encoders), and asymmetric DE student model (67.5M or 11.3M transformer as query encoder and document embeddings inherited from the teacher).}
    \label{tab:msmarco_dede}
    \scalebox{0.85}{
    \begin{tabular}{@{}lccccc@{}}
    \toprule
    \multirow{2}{*}{\textbf{Method}} & \multicolumn{2}{c}{\textbf{Re-ranking}} && \multicolumn{2}{c}{\textbf{Retrieval}} \\
    \cmidrule{2-3} \cmidrule{5-6}
    & \textbf{67.5M} & \textbf{11.3M} && \textbf{67.5M} & \textbf{11.3M} \\
    \midrule
    Train student directly       & 27.0 & 23.0  && 22.6 & 18.6 \\
    \; + Distill from teacher     & 34.6 & 30.4 && 35.0  & 28.6  \\
    \; + Inherit doc embeddings   & 35.2 & 32.1 && 35.7 & 30.3  \\
    \; + Query embedding matching & {36.2} & \best{35.0} && 37.1 & \best{35.4}  \\
    \; + Query generation        &  {36.2} & 34.4 && \best{37.2} & 34.8 \\
    \midrule 
    Train student using only \\
    embedding matching and \\ 
    inherit doc embeddings & \best{36.5} & 33.5 && 36.6 & 31.4  \\
    \; + Query generation   & 36.4 & 34.1 && 36.7 & 32.8 \\
    \bottomrule
    \end{tabular}}
\end{table} 

On top of the two losses (standard distillation and embedding matching), we also use $R_{\rm Emb, Q}({\rm t}, {\rm s}; \QCal')$ from Sec.~\ref{sec:augmentation} on $2$ additional questions (per input question) generated from BART. We also try a variant where we eliminate the standard distillation loss and only employ the embedding matching loss in Eq.~\eqref{eq:emb-loss-de2de-q} along with inheriting teacher's document embeddings. This configuration without the standard distillation loss leads to excellent performance (with query generation again providing additional gains in most cases.)

\begin{table}[t]
    \caption{Average BEIR performance of our DE teacher and \embd~student models and their numbers of trainable parameters. Both models are trained on MSMARCO and evaluated on 14 other datasets (the average does not include MSMARCO). The full table is at Appendix~\ref{appen:beir}. With \embd, student materializes most of the performance of the teacher on the unforeseen datasets.}
    \label{tab:beir_teacher_student_only}
    \centering
    \scalebox{0.85}{
        \renewcommand{\arraystretch}{1.0}
        \begin{tabular}{@{}lc@{}cc@{}}
        \toprule
        \textbf{Method} & \textbf{\#Layers} & \textbf{nDCG@10} & \textbf{R@100}  \\
        \midrule
        DPR~\citep{karpukhin2020dense} & 12 & 22.5 & 47.7 \\
        ANCE~\citep{Xiong:2021ANCE} & 12 & 40.5 & 60.0 \\
        TAS-B~\citep{Hofstatter:2021SIGR} & 6 & 42.8 & 64.8 \\
        GenQ~\citep{thakur2021beir}  & 6 & 42.5 & 64.2 \\
        \midrule
        Our teacher~\citep{reimers-2019-sentence-bert} & 12 (220.2M) & 45.7 & 65.1 \\
        \embd & 6 (67.5M) & 44.0 & 63.5 \\
        \bottomrule
        \end{tabular}
    }
    \vspace{-3mm}
\end{table}

It is worth highlighting that DE models trained with the proposed methods (e.g., asymmetric DE with embedding matching and generation) achieve 99\% of the performance in both NQ/MSMARCO tasks with a query encoder that is 2/3rd the size of that of the teacher. Furthermore, \ even \ with 1/10th
size of the query encoder, our proposal can achieve 95-97\% of the performance. This is particularly useful for latency critical applications with minimal impact on the final performance.

Finally, we take our best student models, i.e., one trained using with additional embedding matching loss and using data augmentation from query generation, and evaluate on test sets.
We compare with various prior work and note that most prior work used considerably bigger models 
in terms of parameters, depth (12 or 24 layers), or width (upto 1024 dims).
For NQ test set results are reported in Table~\ref{tab:nq_testset}, but as MSMARCO does not have any public test set, we instead present results for the BEIR benchmark in Table~\ref{tab:beir_teacher_student_only}. 
Note we also provide evaluation of our SentenceBERT teacher achieving very high performance on the benchmark which can be of independent interest (please refer to~Appendix~\ref{appen:beir} for details).
For both NQ and BEIR, our approach obtains competitive student model with fewer than 50\% of the parameters: even with 6 layers, our student model is very close (98-99\%) to its teacher.

\subsection{CE to DE distillation}
\label{sec:exp-ce2de}

For the CE to DE distillation experiment, we converted SimLM \texttt{[CLS]}-pooled CE model\footnote{https://github.com/microsoft/unilm/tree/master/simlm} to a dual-pooled CE model via standard score-based distillation (cf.~Section~\ref{sec:ir-distill}). We subsequently utilize the resulting dual-pooled version of the SimLM CE model as a teacher to perform CE to DE distillation via embedding alignment. Similar to DE to DE distillation (cf.~Section~\ref{sec:exp-de-to-de-main}), we aim to identify the utility of various components of \embd{} in our exploration. See Table~\ref{tab:msmarco_cede} for the results. 

We also explore distilling dual-pooled CE teacher model to an asymmetric DE student model. In this setting, DE student model simply inherits the document embeddings from the CE teacher model. Crucially, the inheritance of the document embedding from the dual-pooled CE teacher model can be done offline as we feed an \textit{empty query} along with the document (separated by the \texttt{[SEP]} token) to obtain the document embedding from the dual-pooled CE teacher model.

\noindent\textbf{Results and discussion.}
The excellent performance of distillation to an asymmetric DE model (which inherits document embeddings from the dual-pooled CE model) not only showcases the power of embedding alignment via \embd{} but also highlights the effectiveness of dual-pooling method for the teacher. 

\begin{table}
    \centering
    \caption{Performance of various DE models obtained via CE to DE distillation on MSMARCO dev set for \emph{re-ranking} ({original top1000}). The teacher model is a dual-pooled version of the SimLM model which achieves MRR@10 of 40.0 nDCG@10 of 45.8.  The table shows performance of the symmetric DE student model (67.5M or 11.3M transformer as both encoders), and asymmetric DE student model (67.5M or 11.3M transformer as query encoder and document embeddings inherited from the dual-pooled teacher). Note that the document embeddings used during inheritance are generated in a query-independent manner from the CE teacher model (with \textit{empty} query).
    }
    \label{tab:msmarco_cede}
    \vspace{1mm}
    \hspace{-2mm}
    \scalebox{0.85}{
    \begin{tabular}{@{}lccccc@{}}
    \toprule
    \multirow{2}{*}{\textbf{Method}} & \multicolumn{2}{c}{\textbf{MRR@10}} && \multicolumn{2}{c}{\textbf{nDCG@10}} \\
    \cmidrule{2-3} \cmidrule{5-6}
    & \textbf{67.5M} & \textbf{11.3M} && \textbf{67.5M} & \textbf{11.3M} \\
    \midrule
    Train student directly       & 27.0 & 23.0  && 32.2 & 29.7 \\
    \; + Distill from teacher     & 33.2 & 28.6 && 38.7  & 33.6  \\
    \; + Inherit doc embeddings   & 35.4 & 30.2 && 41.0 & 35.6  \\
    \; + Query embedding matching & 36.1 & 31.7 && 41.7 & 37.1  \\
    \; + Query generation        &  {36.3} & 32.1 && {42.0} & 37.6 \\
    \midrule 
    Train student using only \\
    embedding matching and \\ 
    inherit doc embeddings & \best{36.9} & 34.7 && \best{42.6} & 40.4  \\
    \; + Query generation   & {36.8} & \best{35.1} && {42.5} & \best{40.8} \\
    \midrule
    Standard distillation\\
    from \texttt{[CLS]}-teacher    & 32.8 & 28.5 && 38.4 & 33.6  \\
    \bottomrule
    \end{tabular}}
    \vspace{-2mm}
\end{table}

%% file: 020_related_work.tex
\vspace{-2mm}
\section{Related work}
\label{sec:related}
\vspace{-1.5mm}
Here, we position our \embd~work with respect to prior work on distillation and data augmentation for Transformers-based IR models. {We also cover prior efforts on aligning representations during distillation for \emph{non-IR} settings. Unlike our problem setting where the DE student is factorized, these works mainly consider distilling a single large Transformer into a smaller one.} 

\noindent \textbf{Distillation for IR.}
Traditional distillation techniques have been widely applied in the IR literature, often to distill a teacher CE model to a student DE model~\citep{Li:2020Parade,Chen:2021KD}. Recently, distillation from a DE model (with complex late interaction) to another DE model (with inner-product scoring) has also been considered~\citep{Lin:2021Coupled, Hofstatter:2021SIGR}. As for distilling across different model architectures,~\citet{Lu:2020TwinBERT, Izacard:2021Distilling} consider distillation from a teacher CE model to a student DE model. \citet{Hofstatter:2020} conduct an extensive study of knowledge distillation across a wide-range of model architectures. Most existing distillation schemes for IR rely on only teacher scores;
by contrast, we propose a geometric approach that also utilizes the teacher \emph{embeddings}. Many recent efforts~\citep{Qu:2021Rocket, Ren:2021RocketV2, Santhanam:2021} show that iterative multi-stage (self-)distillation improves upon single-stage distillation~\citep{Qu:2021Rocket, Ren:2021RocketV2, Santhanam:2021}. These approaches use a model from the previous stage to obtain labels~\citep{Santhanam:2021} as well as mine harder-negatives~\citep{Xiong:2021ANCE}. We only focus on the single-stage distillation in this paper. Multi-stage procedures are complementary to our work, as one can employ our proposed embedding-matching approach in various stages of such a procedure.
Interestingly, we demonstrate in Sec.~\ref{sec:experiments} that our proposed \embd~can successfully benefit from high quality models trained with such complex procedures~\citep{reimers-2019-sentence-bert, zhang2022ar2}. 
In particular, our single-stage distillation method can transfer almost all of their performance gains to even smaller models.
Also to showcase that our method brings gain orthogonal to how teacher was trained, we conduct experiments with single-stage trained teacher in Appendix~\ref{appendix:old-teacher-results}.

\noindent \textbf{Distillation with representation alignments.}~Outside of the IR context, a few prior works proposed to utilize alignment between hidden layers during distillation~\citep{romero2014fitnets,Sanh:2019DistilBERT,Jiao:2020TinyBERT,aguilar2020knowledge,zhang2020improve}. \citet{chen2022knowledge} utilize the representation alignment to re-use teacher's classification layer for image classification. Unlike these works, our work is grounded in a rigorous theoretical understanding of the teacher-student (generalization) gap for IR models. Further, our work differs from these as it needs to address multiple challenges presented by an IR setting: 1) cross-architecture distillation such as CE to DE distillation; 2) partial representation alignment of query or document representations as opposed to aligning for the entire input, i.e., a query-documents pair; and 3) catering representation alignment approach to novel IR setups such as asymmetric DE configuration. To the best of our knowledge, our work is first in the IR literature that goes beyond simply matching scores (or its proxies) for distillation.

\vspace{-0.5mm}
\noindent \textbf{Semi-supervised learning for IR.}
\newtext{Data augmentation or semi-supervised learning has been previously used to ensure data efficiency in IR~\citep[see, e.g.,][]{MacAvaney:2019,Zhao:2021Distantly}.} \newtext{More interestingly, data augmentation have enabled performance improvements as well.} Doc2query~\citep{Nogueira:2019docTOquery,Nogueira:2019docTTTTTquery} performs document expansion by generating queries that are relevant to the document and appending those queries to the document. Query expansion has also been considered, e.g., for document re-ranking~\citep{Zheng:2020BERT-QE}. Notably, generating synthetic (query, passage, answer) triples from a text corpus to augment existing training data for QA systems also leads to significant gains~\citep{Alberti:2019Synthetic, ouguz2021domain}. 
Furthermore, even zero-shot approaches, where no labeled query-document pairs are used, can also perform competitively to supervised methods~\citep{Lee:2019, Izacard:2021Contriever, Ma:2021ZeroShot, Sachan:2022Improving}. Unlike these works, we utilize query-generation capability to ensure tighter alignment between the embedding spaces of the teacher and student.

\noindent\textbf{{Richer} transformers-based architectures for IR.}~Besides DE and CE models (cf.~Sec.~\ref{sec:bg}), intermediate configurations~\citep{MacAvaney:2020, Khattab:2020, Nie:2020, Luan:2021Sparse} have been proposed. Such models independently encode query and document before applying a more complex \emph{late interaction} between the two. \citet{Nogueira:2020Seq2Seq} explore \textit{generative} encoder-decoder style model for re-ranking. In this paper, we focus on basic DE/CE models to showcase the benefits of our proposed geometric distillation approach. Exploring embedding matching for aforementioned architectures is an interesting avenue for future work.

%% file: 070_conclusion.tex
\vspace{-2mm}
\section{Conclusion}
\label{sec:conclusion}
\vspace{-1.5mm}
We propose \embd~--- a novel distillation method for IR that goes beyond simple score matching. En route, we provide a theoretical understanding of the teacher-student generalization gap in an IR setting which not only motivated \embd~but also inspired new design choices for the student DE models: (a) reusing the teacher's document encoder in the student and (b) aligning query embeddings of the teacher and student. This simple approach delivers consistent quality and computational gains in practical deployments and we demonstrate them on MSMARCO, NQ, and BEIR benchmarks. Finally, we found \embd{} retains 95-97\% of the teacher performance to with 1/10th size students.

\noindent{\bf Limitations.}~As discussed in Sec.~\ref{sec:cetode} and \ref{sec:exp-ce2de}, \embd~requires modifications in the CE scoring function to be effective. In terms of underlying IR model architectures, we only explore Transformer-based models in our experiments; primarily due to their widespread utilization. That said, we expect our results to extend to non-Transformer architectures such as MLPs. Finally, we note that our experiments only consider NLP domains, and exploring other modalities (e.g., vision) or multi-modal settings (e.g., image-to-text search) is left as an interesting avenue for future work.

%% file: 110_appendix.tex
\newpage 

\begin{center}
{\Large EmbedDistill: A Geometric Knowledge Distillation \\ for Information Retrieval (Supplementary)}
\end{center}

\section{Loss functions}\label{appendix:loss}
\begin{figure*}[h]
    \centering
    \begin{subfigure}{0.49\textwidth}
        \centering
        \includegraphics[page=5,width=0.98\textwidth,trim=0 6cm 8cm 0 ,clip]{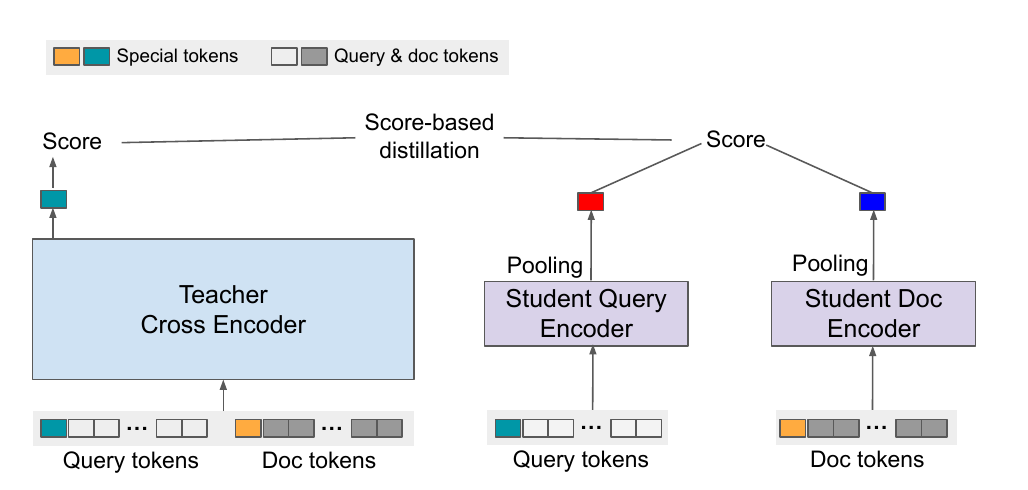}
        \caption{Score-based CE to DE distillation}
        \label{fig:ce2de-score-distill}
    \end{subfigure}
    \hfil
    \begin{subfigure}{0.49\linewidth}
        \centering
        \includegraphics[page=3,width=0.98\textwidth,trim=0 6cm 6cm 0 ,clip]{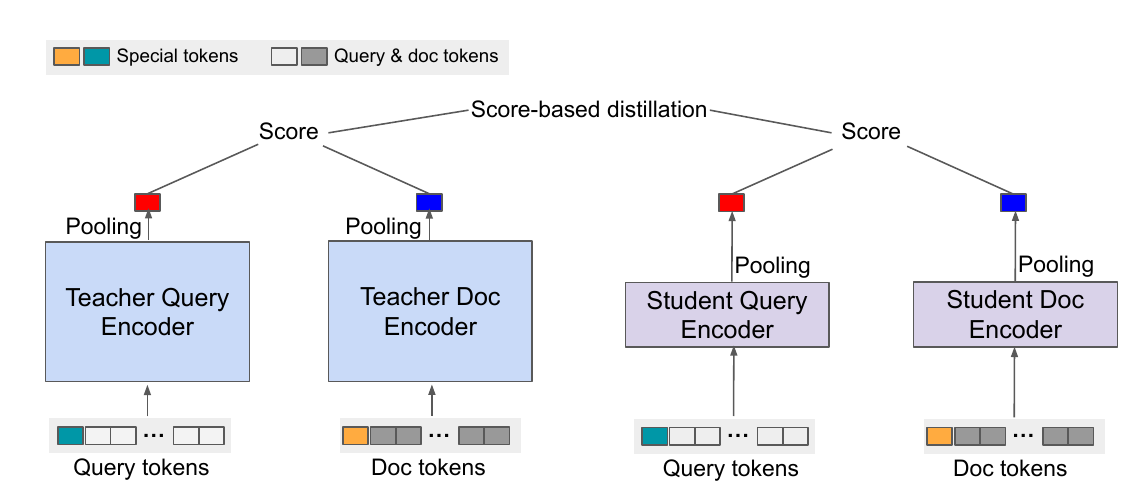}
        \caption{Score-based DE to DE distillation}
        \label{fig:de2de-score-distill}
    \end{subfigure}
    \caption{Illustration of score-based distillation for IR (cf.~Section~\ref{sec:ir-distill}). Fig.~\ref{fig:ce2de-score-distill} describes distillation from a teacher \texttt{[CLS]}-pooled CE model to a student DE model. Fig.~\ref{fig:de2de-score-distill} depicts distillation from a teacher DE model to a student DE model. Here, both distillation setups employ symmetric DE configurations where query and document encoders of the student model are of the same size.
    }
    \label{fig:score-distill}
\end{figure*}

Here, we state various (per-example) loss functions that most commonly define training objectives for IR models. Typically, one hot training with original label is performed using \textit{softmax-based cross-entropy loss} functions:
\begin{align}
&\ell\big(s_{q, \mathbf{d}_i}, \mathbf{y}_i\big) ~ = ~ {-} \sum_{j \in [L]} y_{i,j} \cdot \log \Big( \frac{\exp({s(q_i, d_{i,j})})}{\sum\limits_{j' \in [L]} \exp({s(q_i, d_{i,j'})})}\Big).
\label{eq:ce-one-hot}
\end{align}
Alternatively, it is also common to employ a one-vs-all loss function based on \textit{binary cross-entropy loss} as follows:
\begin{align}
&\ell\big(s_{q, \mathbf{d}_i}, \mathbf{y}_i\big) ~ = ~ {-} \sum_{j \in [L]} \bigg(y_{i,j} \cdot \log \Big(\frac{1}{1~ +~ \exp({-s(q_i, d_{i,j})})}\Big) ~~ + 
\nonumber \\
& \qquad \qquad \qquad \qquad \qquad \qquad
(1 - y_{i,j}) \cdot \log \Big(\frac{1}{1 ~+~ \exp({s(q_i, d_{i,j})})}\Big)\bigg).
\label{eq:binary-ce-one-hot}
\end{align}
{Note that $\mathbf{d}_i = \{d_{i,j}\}_{j \in [L]}$ can be expanded to include various forms of negatives such as in-batch negatives~\citep{karpukhin2020dense} and sampled negatives~\citep{bengio2008sampled}.}

As for distillation (cf.~Fig.~\ref{fig:score-distill}), one can define a distillation objective based on the softmax-based cross-entropy loss as:\footnote{It is common to employ temperature scaling with softmax operation. We do not explicitly show the temperature parameter for ease of exposition.}
\begin{align}
\label{eq:ce-distill}
&\ell_{\rm d}\big(s^{\rm s}_{q, \mathbf{d}_i}, s^{\rm t}_{q, \mathbf{d}_i} \big) =  - \sum_{j \in [L]} \bigg(\frac{\exp({s^{\rm t}_{i,j}})}{\sum_{j' \in [L]} \exp({s^{\rm t}_{i,j'}})} \cdot \log \Big(\frac{\exp({s^{\rm s}_{i,j}})}{\sum_{j' \in [L]} \exp({s^{\rm s}_{i,j'}})}\Big) \bigg),  
\end{align}
where $s^{\rm t}_{i,j} := s^{\rm t}(q_i, d_{i,j})$ and $s^{\rm s}_{i,j} := s^{\rm s}(q_i, d_{i,j})$ denote the teacher and student scores, respectively. On the other hand, the distillation objective with the binary cross-entropy takes the form:
\begin{align}
\ell_{\rm d}\big(s^{\rm s}_{q, \mathbf{d}_i}, s^{\rm t}_{q, \mathbf{d}_i} \big) ~ & = ~ {-} \sum_{j \in [L]} \bigg(\frac{1}{1~ +~ \exp({-s^{\rm t}_{i,j}})} \cdot \log \Big(\frac{1}{1~ +~ \exp({-s^{\rm s}_{i,j}})}\Big) ~ ~~+  \nonumber \\
& \qquad \qquad \qquad \qquad
\frac{1}{1 ~+~ \exp({s^{\rm t}_{i,j}})} \cdot \log \Big(\frac{1}{1 ~+~ \exp({s^{\rm s}_{i,j}})}\Big)\bigg).
\label{eq:binary-ce-distill}
\end{align}
Finally, distillation based on the meas square error (MSE) loss (aka. logit matching) employs the following loss function:
\begin{align}
\label{eq:mse-distill}
\ell_{\rm d}\big(s^{\rm s}_{q, \mathbf{d}_i}, s^{\rm t}_{q, \mathbf{d}_i} \big) 
= \sum\limits_{j \in [L]} \big( s^{\rm t}(q_i, d_{i,j}) - s^{\rm s}(q_i, d_{i,j})\big)^2.
\end{align}

\section{Dual pooling details}\label{appendix:pooling}

In this work, we focus on two kinds of dual pooling strategies:
\begin{itemize}[leftmargin=5mm, itemsep=2mm, partopsep=0pt,parsep=0pt]
\item \textbf{Special tokens-based dual pooling.}~Let ${\rm pool}_{{\rm CLS}}$ and ${\rm pool}_{{\rm SEP}}$ denote the pooling operations that return the embeddings of the \texttt{[CLS]} and \texttt{[SEP]} tokens, respectively. We define
\begin{align}
    {\tt emb}^{\rm t}_{q \leftarrow (q,d)} = {\rm pool}_{{\tt CLS}}\big({\rm Enc}^{\rm t}(\tilde{o}) \big), \nonumber \\
    {\tt emb}^{\rm t}_{d \leftarrow (q,d)} = {\rm pool}_{{\tt SEP}}\big({\rm Enc}^{\rm t}(\tilde{o}) \big), \label{eq:special-dual-pooling}
\end{align}
where $\tilde{o}$ denotes the input token sequence to the Transformers-based encoder, which consists of 
\{ query, document, special \} tokens.
\item \textbf{Segment-based weighted-mean dual pooling.}~Let ${\rm Enc}^{\rm t}(\tilde{o})\vert_{Q}$ and ${\rm Enc}^{\rm t}(\tilde{o})\vert_{D}$ denote the final query token embeddings and document token embeddings produced by the encoder, respectively. We define the \emph{proxy} query and document embeddings
\begin{align}
    {\tt emb}^{\rm t}_{q \leftarrow (q,d)} = {\rm mean_{wt}}\big({\rm Enc}^{\rm t}(\tilde{o})\vert_{Q} \big), \nonumber \\
    {\tt emb}^{\rm t}_{d \leftarrow (q,d)} = {\rm mean_{wt}}\big({\rm Enc}^{\rm t}(\tilde{o})\vert_{D} \big),\label{eq:wm-dual-pooling}
\end{align}
where ${\rm mean_{wt}}(\cdot)$ denotes the weighted mean operation. We employ the specific weighting scheme where each token receives a weight equal to the inverse of the square root of the token-sequence length.
\end{itemize}

\section{Deferred details and proofs from Section~\ref{sec:theory}}
\label{appen:theory}

In this section we present more precise statements and proofs of Theorem~\ref{thm:teacher-student-gap} and Proposition~\ref{prop:lower_complexity_with_frozen_docencoder-main} (stated informally in Section~\ref{sec:theory} of the main text) along with the necessary background. First, for the ease of exposition, we define new notation which will
facilitate theoretical analysis in this section. % and Section~\ref{appen:unif_conv_bd}.

\textbf{Notation.}~Denote the query and document encoders as $f\colon\QCal\to\mathbb{R}^{k}$
and $g\colon\DCal \to\mathbb{R}^{k}$ for the student, and $F\colon\QCal \to\mathbb{R}^{k},G\colon \DCal\to\mathbb{R}^{k}$
for the teacher (in the dual-encoder setting).
With $q$ denoting a query and $d$ denoting a document, $f(q)$ and $g(d)$ then denote
query and document embeddings, respectively, generated by the student. We define $F(q)$ and $G(d)$ similarly for
embeddings by the teacher.\footnote{Note that, as per the notations in the main text, we have $(f, g) = ({\rm Enc}^{\rm s}_{Q}, {\rm Enc}^{\rm s}_{D})$ and $(F, G) = ({\rm Enc}^{\rm t}_{Q}, {\rm Enc}^{\rm t}_{D})$. Similarly, we have $({\tt emb}^{\rm t}_q, {\tt emb}^{\rm t}_d) = (f(q), g(d))$ and $({\tt emb}^{\rm t}_q, {\tt emb}^{\rm t}_d) = (F(q), G(d))$.}

\begin{theorem}[Formal statement of Theorem~\ref{thm:teacher-student-gap}]
\label{thm:teacher-student-gap_appn}
Let $\FCal$ and $\GCal$ denote the function classes for the query and document encoders for the student model, respectively. 
Given $n$ examples $\SCal_n = \{(q_i, d_i, y_i)\}_{i \in [n]} \subset \mathscr{Q} \times \mathscr{D} \times \{0,1\}$,
let $s^{\rm s}(q,d) := s^{f,g}(q_i,d_i) = f(q_i)^Tg(d_i)$ be the scores assigned to the $(q_i,d_i)$ pair by a dual-encoder model with $f \in \FCal$ and $g \in \GCal$ as query and document encoders, respectively. 
Let $\ell$ and $\ell_{\rm d}$ be the binary cross-entropy loss (cf.~Eq.~\eqref{eq:binary-ce-one-hot} with $L=1$) and the distillation-specific loss based on it (cf.~Eq.~\eqref{eq:binary-ce-distill} with $L=1$), respectively. In particular,
\begin{align*}
\ell(s^{F,G}(q_i,d_i), y_i)&:=-y_i\log\sigma\left(F(q_i)^{\top}G(d_i)\right)-(1-y_i)\log\left[1-\sigma\left(F(q_i)^{\top}G(d_i)\right)\right] \nonumber \\
\ell_{\rm d}(s^{f,g}(q_i,d_i),s^{F,G}(q_i,d_i))&:=-\sigma\left(F(q_i)^{\top}G(d_i)\right)\cdot \log\sigma\left(f(q_i)^{\top}g(d_i)\right) \;- \nonumber \\
&\qquad \qquad [1- \sigma\left(F(q_i)^{\top}G(d_i)\right)]\cdot \log\left[1-\sigma\left(f(q_i)^{\top}g(d_i)\right)\right],
\end{align*}
where $\sigma$ is the sigmoid function and $s^{\rm t} := s^{F,G}$ denotes the teacher dual-encoder model with $F$ and $Q$ as its query and document encoders, respectively. 
Assume that 
\begin{enumerate}
\item All encoders $f,g,F,$ and $G$ have the same output dimension. % $k\ge1$.
\item $\exists~K \in(0,\infty)$ such that $\sup_{q\in\mathcal{Q}}\max\left\{ \|f(q)\|_{2},\|F(q)\|_{2}\right\} \le K$
and 
$\sup_{d\in\mathcal{D}}\max\left\{ \|g(d)\|_{2},\|G(d)\|_{2}\right\} \le K$.
\end{enumerate}
Then, we have
\begin{align}
&\underbrace{\mathbb{E}\left[s^{f,g}(q,d)\right]}_{:= R(s^{\rm s}) = R(s^{f, g})} - \underbrace{\mathbb{E}\left[s^{F, G}(q, d)\right]}_{:= R(s^{\rm t}) = R(s^{F, G})} \le \underbrace{\sup_{(f,g) \in \FCal \times \GCal} \left|R(s^{f,g}, s^{F, G}; \SCal_n) - \mathbb{E}\left[\ell_{\rm d} \big( s^{f,g}(q,d), s^{F, G}(q,d)\big)\right]\right|}_{:=\ECal_n(\FCal, \GCal)} \; \nonumber \\ 
&\qquad\qquad \quad + 2K\Big(\underbrace{\frac{1}{n}\sum_{i \in [n]}\|g(d_i)-G(d_i)\|_{2}}_{:=R_{{\rm Emb}, D}({\rm t}, {\rm s}; \SCal_n)} \; + 
 \underbrace{\frac{1}{n}\sum_{i\in[n]}\|f(q_i)-F(q_i)\|_{2}\Big)}_{:=R_{{\rm Emb}, Q}({\rm t}, {\rm s}; \SCal_n)} \; + \underbrace{R(s^{F, G}; \SCal_n) - R(s^{F, G})}_{:=\Delta(s^{\rm t}; \SCal_n)}  \nonumber \\
&\qquad \qquad \quad + K^2\Big(\mathbb{E}\left[\left|\sigma(F(q)^{\top}G(d))-y\right|\right] + \frac{1}{n}\sum_{i \in [n]}\left|\sigma\left(F(q_i)^{\top}G(d_i)\right)-y_i\right|\Big).
\end{align}
\end{theorem}

\begin{proof}
Note that
\begingroup
\allowdisplaybreaks
\begin{align}
&R(s^{f,g}) - R(s^{F, G}) = R(s^{f,g}) - R(s^{f,g}, s^{F,G}) + R(s^{f,g}, s^{F,G}) - R(s^{F, G}) \nonumber \\
&\qquad \overset{(a)}{\leq} K^2\mathbb{E}\left[\left|\sigma(F(q)^{\top}G(d))-y\right|\right] + R(s^{f,g}, s^{F,G}) - R(s^{F, G}) \nonumber \\
&\qquad = K^2\mathbb{E}\left[\left|\sigma(F(q)^{\top}G(d))-y\right|\right] + R(s^{f,g}, s^{F,G}) - R(s^{f,g}, s^{F,G}; \SCal_n)  \; + \nonumber \\
&\qquad \qquad R(s^{f,g}, s^{F,G}; \SCal_n)  - R(s^{F, G}) \nonumber \\
&\qquad \overset{(b)}{\leq} K^2\mathbb{E}\left[\left|\sigma(F(q)^{\top}G(d))-y\right|\right] + \ECal_{n}(\FCal, \GCal) + R(s^{f,g}, s^{F,G}; \SCal_n)  - R(s^{F, G}) \nonumber \\
&\qquad =  K^2\mathbb{E}\left[\left|\sigma(F(q)^{\top}G(d))-y\right|\right] + \ECal_{n}(\FCal, \GCal) + R(s^{f,g}, s^{F,G}; \SCal_n) - R(s^{F, G}; \SCal_n)  \; + \nonumber \\
&\qquad \qquad R(s^{F, G};\SCal_n) - R(s^{F, G}) \nonumber \\
&\qquad \overset{(c)}{\leq}  K^2\mathbb{E}\left[\left|\sigma(F(q)^{\top}G(d))-y\right|\right] + \ECal_{n}(\FCal, \GCal) + \underbrace{R(s^{F, G};\SCal_n) - R(s^{F, G})}_{:=\Delta(s^{\rm t}; \SCal_n)} \; + \nonumber \\
&\qquad \qquad \frac{2K}{n}\sum_{i \in [n]}\|g(d_i)-G(d_i)\|_{2} \; + 
 \frac{2K}{n}\sum_{i\in[n]}\|f(q_i)-F(q_i)\|_{2} \; + \nonumber \\
&\qquad \qquad \frac{K^2}{n}\sum_{i \in [n]}\left|\sigma\left(F(q_i)^{\top}G(d_i)\right)-y_i\right| 
\end{align}
\endgroup
where $(a)$ follows from Lemma~\ref{lem:bin_ce_distill_vs_onehot_student}, $(b)$ follows from the definition of $\ECal_n(\FCal, \GCal)$, and $(c)$ follows from Proposition~\ref{prop:bin_ce_risk_bound_n}.
\end{proof}

\subsection{Bounding the difference between student's empirical \emph{distillation} risk and teacher's empirical risk}
\label{sec:bin_ce_risk}

\begin{lemma}
\label{prop:bin_ce_risk_bound_n}
Given $n$ examples $\SCal_n = \{(q_i, d_i, y_i)\}_{i \in [n]} \subset \mathscr{Q} \times \mathscr{D} \times \{0,1\}$,
let $s^{f,g}(q_i,d_i) = f(q_i)^Tg(d_i)$ be the scores assigned to the $(q_i,d_i)$ pair by a dual-encoder model with $f$ and $g$ as query and document encoders, respectively.
Let $\ell$ and $\ell_{\rm d}$ be the binary cross-entropy loss (cf.~Eq.~\eqref{eq:binary-ce-one-hot} with $L=1$) and the distillation-specific loss based on it (cf.~Eq.~\eqref{eq:binary-ce-distill} with $L=1$), respectively. In particular,
\begin{align*}
\ell(s^{F,G}(q_i,d_i), y_i)&:=-y_i\log\sigma\left(F(q_i)^{\top}G(d_i)\right)-(1-y_i)\log\left[1-\sigma\left(F(q_i)^{\top}G(d_i)\right)\right] \nonumber \\
\ell_{\rm d}(s^{f,g}(q_i,d_i),s^{F,G}(q_i,d_i))&:=-\sigma\left(F(q_i)^{\top}G(d_i)\right)\cdot \log\sigma\left(f(q_i)^{\top}g(d_i)\right) \;- \nonumber \\
&\qquad \qquad [1- \sigma\left(F(q_i)^{\top}G(d_i)\right)]\cdot \log\left[1-\sigma\left(f(q_i)^{\top}g(d_i)\right)\right],
\end{align*}
where $\sigma$ is the sigmoid function and $s^{F,G}$ denotes the teacher dual-encoder model with $F$ and $Q$ as its query and document encoders, respectively. 
Assume that 
\begin{enumerate}
\item All encoders $f,g,F,$ and $G$ have the same output dimension $k\ge1$.
\item $\exists~K \in(0,\infty)$ such that $\sup_{q\in\mathcal{Q}}\max\left\{ \|f(q)\|_{2},\|F(q)\|_{2}\right\} \le K$
and 
$\sup_{d\in\mathcal{D}}\max\left\{ \|g(d)\|_{2},\|G(d)\|_{2}\right\} \le K$.
\end{enumerate}
Then, we have
\begin{align}
\label{eq:s-dist-vs-t-onehot_n}
&\frac{1}{n}\sum_{i \in [n]}\ell_{\rm d}\big(s^{f,g}(q_i,d_i),s^{F,G}(q_i,d_i)\big) -\frac{1}{n}\sum_{i \in [n]}\ell\big(s^{F,G}(q_i,d_i), y_i\big)  \le \; \nonumber \\ 
&\qquad\qquad \qquad \qquad \frac{2K}{n}\sum_{i \in [n]}\|g(d_i)-G(d_i)\|_{2} \; + 
 \frac{2K}{n}\sum_{i\in[n]}\|f(q_i)-F(q_i)\|_{2}
 \; + \nonumber \\
&\qquad \qquad \qquad \qquad \qquad
\frac{K^2}{n}\sum_{i \in [n]}\left|\sigma\left(F(q_i)^{\top}G(d_i)\right)-y_i\right|.
\end{align}
\end{lemma}

\begin{proof}
We first note that the distillation loss can be rewritten as 
\begin{align*}
\ell_{\rm d}\big(s^{f,g}(q,d),s^{F,G}(q,d)\big) & =\left(1-\sigma(F(q)^{\top}G(d)\right)f(q)^{\top}g(d)+\gamma(-f(q)^{\top}g(d)),
\end{align*}
where $\gamma(v):=\log[1+e^{v}]$ is the softplus function. Similarly,
the one-hot (label-dependent) loss can be rewritten as
\begin{align*}
\ell\big(s^{F,G}(q,d), y\big) & =(1-y)F(q)^{\top}G(d)+\gamma(-F(q)^{\top}G(d)).
\end{align*}
Recall from our notation in Section~\ref{sec:bg} that
\begin{align}
R(s^{f,g}, s^{F, G}; \SCal_n) & :=\frac{1}{n}\sum_{i \in [n]}\ell_{\rm d}\big(s^{f,g}(q_i,d_i),s^{F,G}(q_i,d_i)\big), \label{eq:tildeRfg_n}\\
R(s^{F,G}; \SCal_n) & := \frac{1}{n}\sum_{i \in [n]}\ell\big(s^{F,G}(q_i,d_i), y_i\big), \label{eq:RFG_n}
\end{align}
as the empirical risk based on the distillation loss, and the empirical
risk based on the label-dependent loss, respectively. With this notation,
the quantity to upper bound can be rewritten as 
\begin{align}
&R(s^{f,g}, s^{F, G}; \SCal_n) - R(s^{F,G}; \SCal_n) = \underbrace{R(s^{f,g}, , s^{F, G}; \SCal_n)-R(s^{f,G}, s^{F, G}; \SCal_n)}_{:=\square_{1}}+ \nonumber \\
&\qquad \qquad \quad \underbrace{R(s^{f,G}, s^{F, G}; \SCal_n)- R(s^{F,G}, s^{F, G}; \SCal_n)}_{:=\square_{2}}+\underbrace{R(s^{F,G}, s^{F, G}; \SCal_n)-R(s^{F,G}; \SCal_n)}_{:=\square_{3}}.\label{eq:bce_risk_diff_n}
\end{align}
We start by bounding $\square_{1}$ as
\begingroup
\allowdisplaybreaks
\begin{align}
\square_{1} 
 & =\frac{1}{n}\sum_{i \in [n]}\Big(\ell_{\rm d}\big(s^{f,g}(q_i,d_i),s^{F,G}(q_i,d_i)\big)- \ell_{\rm d}\big(s^{f,G}(q_i,d_i),s^{F,G}(q_i,d_i)\big)\Big)\nonumber \\
 & =\frac{1}{n}\sum_{i \in [n]}\Big(\left(1-\sigma(F(q_i)^{\top}G(d_i))\right)f(q_i)^{\top}g(d_i)+\gamma(-f(q_i)^{\top}g(d_i)) \nonumber \\
 & \qquad \qquad -\left(1-\sigma(F(q_i)^{\top}G(d_i))\right)f(q_i)^{\top}G(d_i) - \gamma(-f(q_i)^{\top}G(d_i))\Big)\nonumber \\
 & =\frac{1}{n}\sum_{i \in [n]}\Big(f(q_i)^{\top}\big(g(d_i)-G(d_i)\big)\left(1-\sigma(F(q_i)^{\top}G(d_i))\right) \;  \nonumber \\ &\qquad \qquad  + \gamma(-f(q_i)^{\top}g(d_i))-\gamma(-f(q_i)^{\top}G(d_i))\Big)\nonumber \\
 & \stackrel{(a)}{\le}\frac{1}{n}\sum_{i \in [n]}\Big(f(q_i)^{\top}\big(g(d_i)-G(d_i)\big)\left(1-\sigma(F(q_i)^{\top}G(d_i))\right)+\left|f(q_i)^{\top}g(d_i)-f(q_i)^{\top}G(d_i)\right|\Big)\nonumber \\
 & \stackrel{(b)}{\le}\frac{1}{n}\sum_{i \in [n]}\Big(\|f(q_i)\|\|g(d_i)-G(d_i)\|\left(1-\sigma(F(q_i)^{\top}G(d_i))\right)+\|f(q_i)\|\|g(d_i)-G(d_i)\|\Big)\nonumber \\
 & \stackrel{}{\le}\frac{K}{n}\sum_{i \in [n]}\|g(d_i)-G(d_i)\|_{2}\left(2-\sigma(F(q_i)^{\top}G(d_i))\right)\Big)\nonumber \\
 & \le\frac{2K}{n}\sum_{i \in [n]}\|g(d_i)-G(d_i)\|_{2},\label{eq:bce_diff1_n}
\end{align}
\endgroup
where at $(a)$ we use the fact that $\gamma$ is a Lipschitz continuous
function with Lipschitz constant 1, and at $(b)$ we use Cauchy-Schwarz
inequality.

Similarly for $\square_{2}$, we proceed as 
\begingroup
\allowdisplaybreaks
\begin{align}
\square_{2}  &= \frac{1}{n}\sum_{i \in [n]}\Big(\ell_{\rm d}\big(s^{f,G}(q_i,d_i),s^{F,G}(q_i,d_i)\big)- \ell_{\rm d}\big(s^{F,G}(q_i,d_i),s^{F,G}(q_i,d_i)\big)\Big)\nonumber \\
 & =\frac{1}{n}\sum_{i \in [n]}\Big(\left(1-\sigma(F(q_i)^{\top}G(d_i))\right)f(q_i)^{\top}G(d_i)+\gamma(-f(q_i)^{\top}G(d_i))\nonumber \\
 & \qquad \qquad -\left(1-\sigma(F(q_i)^{\top}G(d_i))\right)F(q_i)^{\top}G(d_i) - \gamma(-F(q_i)^{\top}G(d_i))\Big)\nonumber \\
 & =\frac{1}{n}\sum_{i \in [n]}\Big(G(d_i)^{\top}(f(q_i)-F(q_i))\left(1-\sigma(F(q_i)^{\top}G(d_i))\right) \;  \nonumber \\ &\qquad \qquad  + \gamma(-f(q_i)^{\top}G(d_i))-\gamma(-F(q_i)^{\top}G(d_i))\Big)\nonumber \\
 & \le\frac{1}{n}\sum_{i \in [n]}\Big(\|G(d_i)\|\|f(q_i)-F(q_i)\|+\left|f(q_i)^{\top}G(d_i)-F(q_i)^{\top}G(d_i)\right|\Big) \nonumber \\
 & \le\frac{2K}{n}\sum_{i \in [n]}\|f(q_i)-F(q_i)\|_{2}.\label{eq:bce_diff2_n}
\end{align}
\endgroup
$\square_{3}$ can be bounded as
\begingroup
\allowdisplaybreaks
\begin{align}
\square_{3} & = R(s^{F,G}, s^{F, G}; \SCal_n)-R(s^{F,G}; \SCal_n)\nonumber \\
 & =\frac{1}{n}\sum_{i \in [n]}\Big(\ell_{\rm d}\big(s^{F,G}(q_i,d_i),s^{F,G}(q_i,d_i)\big) -\ell\big(s^{F,G}(q_i,d_i), y_i\big)\Big)\nonumber \\
 & =\frac{1}{n}\sum_{i \in [n]}\Big(\left(1-\sigma(F(q_i)^{\top}G(d_i))\right)F(q_i)^{\top}G(d_i)+\gamma(-F(q_i)^{\top}G(d_i))\nonumber \\
 & \qquad \qquad -(1-y_i)F(q_i)^{\top}G(d_i) - \gamma(-F(q_i)^{\top}G(d_i))\Big)\nonumber \\
 & =\frac{1}{n}\sum_{i \in [n]}\Big(\big( 1-\sigma(F(q_i)^{\top}G(d_i))-(1-y_i)\big) F(q_i)^{\top}G(d_i)\Big)\nonumber \\
 & \le\frac{K^2}{n}\sum_{i \in [n]}\left|\sigma(F(q_i)^{\top}G(d_i))-y_i\right|.\label{eq:bce_diff3_n}
\end{align}
\endgroup
Combining Eq.~\ref{eq:bce_risk_diff_n}, \ref{eq:bce_diff1_n}, \ref{eq:bce_diff2_n},
and \ref{eq:bce_diff3_n} establishes the bound in Eq.~\ref{eq:s-dist-vs-t-onehot_n}.
\end{proof}

\begin{lemma}
\label{lem:bin_ce_distill_vs_onehot_student}
Given an example $(q, d, y) \in \mathscr{Q} \times \mathscr{D} \times \{0,1\}$,  
let $s^{f,g}(q,d) = f(q)^Tg(d)$ be the scores assigned to the $(q,d)$ pair by a dual-encoder model with $f$ and $g$ as query and document encoders, respectively. Let $\ell$ and $\ell_{\rm d}$ be the binary cross-entropy loss (cf.~Eq.~\eqref{eq:binary-ce-one-hot} with $L=1$) and the distillation-specific loss based on it (cf.~Eq.~\eqref{eq:binary-ce-distill} with $L=1$), respectively. In particular,
\begin{align*}
\ell(s^{f,g}(q,d), y)&:=-y\log\sigma\left(f(q)^{\top}g(d)\right)-(1-y)\log\left[1-\sigma\left(f(q)^{\top}g(d)\right)\right] \nonumber \\
\ell_{\rm d}(s^{f,g}(q,d),s^{F,G}(q,d))&:=-\sigma\left(F(q)^{\top}G(d)\right)\cdot \log\sigma\left(f(q)^{\top}g(d)\right) \;- \nonumber \\
&\qquad \qquad [1- \sigma\left(F(q)^{\top}G(d)\right)]\cdot \log\left[1-\sigma\left(f(q)^{\top}g(d)\right)\right],
\end{align*}
where $\sigma$ is the sigmoid function and $s^{F,G}$ denotes the teacher dual-encoder model with $F$ and $Q$ as its query and document encoders, respectively.
Assume that 
\begin{enumerate}
\item All encoders $f,g,F,$ and $G$ have the same output dimension $k\ge1$.
\item $\exists~K \in(0,\infty)$ such that $\sup_{q\in\mathcal{Q}}\max\left\{ \|f(q)\|_{2},\|F(q)\|_{2}\right\} \le K$
and 
$\sup_{d\in\mathcal{D}}\max\left\{ \|g(d)\|_{2},\|G(d)\|_{2}\right\} \le K$.
\end{enumerate}
Then, we have
\begin{align}
\label{eq:s-onehot-vs-t-onehot}
&\underbrace{\mathbb{E}\big[\ell\big(s^{f,g}(q,d), y\big)\big]}_{:=R(s^{f,g})} - \underbrace{\mathbb{E}\left[\ell_{\rm d}\big(s^{f,g}(q,d),s^{F,G}(q,d)\big)\right]}_{:=R(s^{f,g}, s^{F, G})} 
\le \; K_{Q}K_{D}\mathbb{E}\left[\left|\sigma(F(q)^{\top}G(d))-y\right|\right] 
\end{align}
where expectation are defined by a joint distribution $\mathbb{P}(q, d, y)$ over $ \mathscr{Q} \times \mathscr{D} \times \{0,1\}$
\end{lemma}
\begin{proof}
Similar to the proof of Proposition~\ref{prop:bin_ce_risk_bound_n}, we utilize the fact that
\begin{align*}
\ell\big(s^{F,G}(q,d), y\big) & =(1-y)F(q)^{\top}G(d)+\gamma(-F(q)^{\top}G(d)), \\
\ell_{\rm d}\big(s^{f,g}(q,d),s^{F,G}(q,d)\big) & =\left(1-\sigma(F(q)^{\top}G(d)\right)f(q)^{\top}g(d)+\gamma(-f(q)^{\top}g(d)),
\end{align*}
where $\gamma(v):=\log[1+e^{v}]$ is the softplus function. Now, 
\begin{align}
&\mathbb{E}\left[\ell\big(s^{f,g}(q,d), y\big) - \ell_{\rm d}\big(s^{f,g}(q,d),s^{F,G}(q,d)\big)\right]  \\
 & \qquad \qquad \qquad {=}\mathbb{E}\left[(1-y)f(q)^{\top}g(d)+\gamma(-f(q)^{\top}g(d))\right]\nonumber \\
 & \qquad \qquad \qquad \phantom{=} - \mathbb{E}\left[\left(1-\sigma(F(q)^{\top}G(d))\right)f(q)^{\top}g(d)+\gamma(-f(q)^{\top}g(d))\right]\nonumber \\
 & \qquad \qquad \qquad=\mathbb{E}\left[\Big(1 - y - \big(1-\sigma(F(q)^{\top}G(d))\big)\Big) F(q)^{\top}G(d)\right]\nonumber \\
 & \qquad \qquad \qquad \le K^2\mathbb{E}\left[\left|\sigma(F(q)^{\top}G(d))-y\right|\right],\label{eq:bce_diff4}  
\end{align}
which completes the proof.
\end{proof}

\subsection{Uniform deviation bound}
\label{appen:unif_conv_bd}
Let $\FCal$ denote the class of functions that map queries in $\QCal$ to their embeddings in $\R^k$ via the query encoder. Define $\GCal$ analogously for the doc encoder, which consists of functions that map documents in $\DCal$ to their embeddings in $\R^k$. To simplify exposition, we assume that each training example consists of a single relevant or irrelevant document for each query, i.e., $L= 1$ in Section~\ref{sec:bg}. 
Let 
\[
\FCal \GCal = \{ (q,d) \mapsto f(q)^\top g(d) 
\;\mid\; f \in \FCal, g \in \GCal\}
\]
Given $\SCal_n = \{(q_i,d_i, y_i) : i\in [n]\}$, let $N(\epsilon, \HCal)$ denote the $\eps$-covering number of a function class $\HCal$ with respect to $L_2(\mathbb{P}_n)$ norm, where 
$\|h\|_{L_2(\mathbb{P}_n)}^2 := \|h\|_n^2 := \frac 1n \sum_{i=1}^n \| h(q_i, d_i) \|_2^2$. Depending on the context, the functions in $\HCal$ may map to $\R$ or $\R^d$.
\begin{proposition}
\label{prop:lower_complexity_with_frozen_docencoder}
Let $s^{\mathrm{t}}$ be scorer of a teacher model and $\ell_{\rm d}$ be a distillation loss function which is $L_{\ell_{\rm d}}$-Lipschitz in its first argument. Let the embedding functions in $\FCal$ and $\GCal$ output vectors with  $\ell_2$ norms at most $K$. 
Define the uniform deviation
\[
\ECal_n (\FCal, \GCal) = \sup_{f \in \FCal, g \in \GCal} \left|\frac{1}{n}\sum\nolimits_{i \in [n]} \ell_{\rm d}\big( f(q_i)^\top g(d_i), s_{q_i,d_i}^{\rm t}\big) - \mathbb{E}_{q,d} \ell_{\rm d} \big( f(q)^\top g(d), s_{q,d}^{\rm t}\big)\right|.
\]
For any $g^* \in \GCal$, we have
\begin{align*}
  \mathbb{E}_{\SCal_n} \ECal_n (\FCal, \GCal)
  &\leq
  \mathbb{E}_{\SCal_n}\frac{48 K L_{\ell_{\rm d}}}{\sqrt{n}} \int_{0}^\infty \sqrt{\log N(u, \FCal) + \log N(u, \GCal)} \;du,\\
  \mathbb{E}_{\SCal_n} \ECal_n (\FCal, \{g^*\}) &\leq
  \mathbb{E}_{\SCal_n} \frac{48 K L_{\ell_{\rm d}}}{\sqrt{n}} \int_{0}^\infty \sqrt{\log N(u, \FCal)} \;du.
\end{align*}
\end{proposition}

\begin{proof}[Proof of Proposition~\ref{prop:lower_complexity_with_frozen_docencoder}]
We first symmetrize excess risk to get Rademacher complexity, then bound the Rademacher complexity with Dudley's entropy integral.

For a training set $\SCal_n$, the empirical Rademacher complexity of a class of functions $\HCal$ that maps $\QCal \times \DCal$ to $\R$ is defined by
\[
\rad_n(\HCal) = \mathbb{E}_\sigma \sup_{h \in \HCal} \frac{1}{n} \sum_{i=1}^n \varepsilon_i h(q_i, d_i),
\]
where $\{\varepsilon_i\}$ denote i.i.d. Rademacher random variables taking the value in $\{+1, -1\}$ with equal probability.
By symmetrization~\citep{Bousquet2004} and the fact that $\ell_{\rm d}$ is $L_{\ell_{\rm d}}$-Lipschitz in its first argument, we get
\[
E_{\SCal_n} \ECal_n (\FCal, \GCal) \leq 2 L_{\ell_{\rm d}} \mathbb{E}_{\SCal_n} \rad_n(\FCal\GCal).
\]
Then, Dudley's entropy integral \citep[see, e.g.,][]{ledoux1991probability} gives
\[
\rad_n(\FCal\GCal) \leq \frac{12}{\sqrt{n}} \int_0^\infty \sqrt{\log N(u, \FCal \GCal)} \; du. 
\]
From Lemma~\ref{lem:cov_num_prod_class} with $K_Q = K_D = K$, for any $u > 0$,
\[
N(u, \FCal\GCal) \leq N \left(\frac{u}{2K}, \FCal \right) N\left(\frac{u}{2K}, \GCal \right).
\]
Putting these together,
\begin{align}
\label{eq:rad-prop-1}
\mathbb{E}_{\SCal_n} \ECal_n (\FCal, \GCal) \leq \frac{24 L_{\ell_{\rm d}}}{\sqrt{n}} \int_0^\infty \sqrt{\log N(u/2K, \FCal) + \log N(u/2K, \GCal)} \;du.
\end{align}
Following the same steps with $\GCal$ replaced by $\{g^*\}$, we get
\begin{align}
\label{eq:rad-prop-2}
\mathbb{E}_{\SCal_n} \ECal_n (\FCal,\{g^*\} ) \leq \frac{24 L_{\ell_{\rm d}}}{\sqrt{n}} \int_0^\infty \sqrt{\log N(u/2K, \FCal)} \;du
\end{align}
By changing variable in Eq.~\eqref{eq:rad-prop-1} and Eq.~\eqref{eq:rad-prop-2}, we get the stated bounds.
\end{proof}

For $f: \QCal \to \R^k, g: \DCal \to \R^k$, 
define $f g: \QCal \times \DCal \to \R$ by $f g(q,d) = f(q)^\top g(d).$

\begin{lemma}
\label{lem:cov_num_prod_class}
Let $f_1,\dots,f_N$ be an $\eps$-cover of $\FCal$ and 
$g_1,\dots,g_M$ be an $\eps$-cover of $\GCal$ in $L_2(\mathbb{P}_n)$ norm.
Let $\sup_{f\in \FCal} \sup_{q\in \QCal} \|f(q) \|_2 \leq K_Q$ and
$\sup_{g\in \GCal} \sup_{d\in \DCal} \|g(d) \|_2 \leq K_D $.
Then, 
\[
\{ f_i g_j \mid i\in [N], j \in [M]\}
\]
is a $(K_Q + K_D) \eps$-cover of $\FCal \GCal$.
\end{lemma}
\begin{proof}[Proof of Lemma~\ref{lem:cov_num_prod_class}]
For arbitrary $f \in \FCal, g \in \GCal$, there exist
$\tilde f  \in \{f_1, \dots, f_N\}, \tilde g \in \{g_1, \dots, g_M\}$ such that $ \| f - \tilde f \|_n \leq \eps, \| g - \tilde g \|_n \leq \eps.$ It is sufficient to show that 
$\| f g - \tilde f \tilde g\|_n \leq (K_Q + K_D) \eps.$
Decomposing using triangle inequality,
\begin{align}
    \| f g - \tilde f \tilde g \|_n
    &= 
    \| f g - f \tilde g + f \tilde g - \tilde f \tilde g \|_n \nonumber
    \\
    \label{eq:fg_decomposition}
    &\leq
    \| f g - f \tilde g \|_n + 
    \| f \tilde g - \tilde f \tilde g \|_n.
\end{align}
To bound the first term, using Cauchy-Schwartz inequality, we can write
\[
    \frac{1}{n} \sum_{i=1}^n \left( f(q_i)^\top g(d_i) - \tilde f(q_i)^\top \tilde g(d_i) \right)^2
    \leq 
    \sup_{q \in \QCal} \| f(q) \|_2^2 \cdot \frac{1}{n} \sum_{i=1}^n \| (g - \tilde g) (d_i) \|_2^2.
    \nonumber
\]
Therefore 
\[
    \| f g - f \tilde g \|_n \leq K_Q \| g - \tilde g\|_n \leq K_Q \eps.
\]
Similarly
\[
    \| f \tilde g - \tilde f \tilde g \|_n \leq K_D \| f - \tilde f \|_n \leq K_D \eps
\]
Plugging these in Eq.~\eqref{eq:fg_decomposition}, we get
\[
\| f g - \tilde f \tilde g\|_n \leq (K_Q + K_D) \eps.
\]
This completes the proof.
\end{proof}

\section{Evaluation metric details}
\label{appendix:evaluation}

For NQ, we evaluate models with full \emph{strict} recall metric, meaning that the model is required to find a \emph{golden} passage from the whole set of candidates (21M). Specifically, for $k \geq 1$, recall@$k$ or R@$k$ denotes the percentage of questions for which the associated golden passage is among the $k$ passages that receive the highest relevance scores by the model. In addition, we also present results for \emph{relaxed} recall metric considered by \citet{Karpukhin:2020}, where R@$k$ denotes the percentage of questions where the corresponding answer {string}
is present in at least one of the $k$ passages with the highest model (relevance) scores. 

For both MSMARCO retrieval and re-ranking tasks, we follow the standard evaluation metrics \emph{Mean Reciprocal Rank}(MRR)@10 and \emph{normalized Discounted Cumulative Gain} (nDCG)@10. For retrieval tasks, these metrics are computed with respect to the whole set of candidates passages (8.8M). On the other hand, for re-ranking task, the metrics are computed with respect to BM25 generated 1000 candidate passages --\emph{the originally provided}-- for each query. 
{
Please note that some papers use more powerful models (e.g., DE models) to generate the top 1000 candidate passages, which is not a standard re-ranking evaluation and should not be compared directly.
}
We report 100 $\times$ MRR@10 and 100 $\times$ nDCG@10, as per the convention followed in the prior works. 

\section{Query generation details}
\label{appendix:query-generation}

We introduced query generation to encourage geometric matching in local regions, which can aid in transferring more knowledge in confusing neighborhoods. As expected, this further improves the distillation effectiveness on top of the embedding matching in most cases. To focus on the local regions, we generate queries from the observed examples by adding local perturbation in the data manifold (embedding space). Specifically, we employ an off-the-shelf encoder-decoder model -- BART-base~\cite{Lewis:2020BART}. First, we embed an observed query in the corresponding dataset. Second, we add a small perturbation to the query embedding. Finally, we decode the perturbed embedding to generate a new query in the input space. Formally, the generated query $x'$ given an original query $x$ takes the form $ x' = \mathrm{Dec}( \mathrm{Enc}(x) + \epsilon ) $, where $\mathrm{Enc}()$ and $\mathrm{Dec}()$ correspond to the encoder and the decoder from the off-the-shelf model, respectively, and $\epsilon$ is an isotropic Gaussian noise. Furthermore, we also randomly mask the original query tokens with a small probability. We generate two new queries from an observed query and use them as additional data points during our distillation procedure.

As a comparison, we tried adding the same size of random sampled queries instead of the ones generated via the method described above. That did not show any benefit, which justifies the use of our query/question generation method.

\section{Experimental details and additional results}
\label{appendix:additional-exps}

\subsection{Additional training details}\label{appen:additional-training-details}
\noindent\textbf{Optimization.}~For all of our experiments, we %train all models with 
use ADAM weight decay optimizer with a short warm up period and a linear decay schedule. We use the initial learning rate of $10^{-5}$ and $2.8\times10^{-5}$ for experiments on
NQ and MSMARCO, respectively. We chose batch sizes to be $128$.

\subsection{Additional results on NQ}
\label{appendix:additional-exps-nq}

See Table~\ref{tab:nq-dev-relaxed} for the performance of various DE models on NQ, as measured by the \textit{relaxed} recall metric.

% New NQ DE->DE table WIP
\begin{table*}[h]
    \centering
    \caption{\textit{Relaxed} recall performance of various student DE models on NQ dev set, including symmetric DE student model (67.5M or 11.3M transformer for both encoders), and asymmetric DE student model (67.5M or 11.3M transformer as query encoder and document embeddings inherited from the teacher). All distilled students used the same teacher (110M parameter BERT-base models as both encoders), with the performance (in terms of relaxed recall) of Recall@5 = 87.2, Recall@20 = 94.7, Recall@100 = 98.1. \emph{Note: the proposed method can achieve 100\% of teacher's performance even with 2/3$^{rd}$ size of the query encoder, and 92-97\% with even 1/10$^{th}$ size.}}
    \label{tab:nq-dev-relaxed}
    \vskip 0.1in
    \scalebox{0.9}{
    \begin{tabular}{@{}lcccccccc}
    \toprule
    \multirow{2}{*}{\textbf{Method}} & \multicolumn{2}{c}{\textbf{Recall@5}} && \multicolumn{2}{c}{\textbf{Recall@20}} && \multicolumn{2}{c}{\textbf{Recall@100}} \\
    \cmidrule{2-3} \cmidrule{5-6}  \cmidrule{8-9} 
    & \textbf{67.5M} & \textbf{11.3M} && \textbf{67.5M} & \textbf{11.3M} && \textbf{67.5M} & \textbf{11.3M} \\
    \midrule
    Train student directly              & 62.5 & 49.7 && 82.5 & 73.0 && 93.7 & 88.2 \\
    \quad + Distill from teacher        & 82.7 & 66.1 && 92.9 & 84.0 && 97.3 & 93.1 \\
    \quad + Inherit document embeddings & 84.7 & 73.0 && 93.7 & 85.4 && 97.6 & 93.3 \\
    \quad + Query embedding matching    & 87.2 & 77.6 && \best{95.0} & 88.0 && 97.9 & 94.3 \\
    \quad + Query generation            & \best{87.8} & \best{80.3} && 94.8 & \best{89.9} && \best{98.0} & \best{95.6} \\
    \midrule
    Train student only using embedding \\
    matching and inherit doc embeddings & 86.4 & 69.1 && 94.2 & 81.6 && 97.7 & 89.9 \\
    \quad + Query generation            & 86.7 & 72.9 && 94.4 & 84.9 && 97.8 & 92.2 \\
    \bottomrule
    \end{tabular}}
\end{table*}

\subsection{Additional results on MSMARCO}
\label{appendix:additional-exps-msmarco}

\subsubsection{DE to DE distillation}

\noindent See Table~\ref{tab:msmarco_dede_ndcg} for DE to DE distillation results on MSMARCO retrieval and re-ranking task, as measured by the nDCG@10 metric (see Section~\ref{sec:exp-de-to-de-main} for the results on MRR@10 metric).

\begin{table}[h!]
    \vspace{-3mm}
    \centering
    \caption{Performance of various DE models on MSMARCO dev set for both \emph{re-ranking} ({original top1000}) and \emph{retrieval} tasks ({full corpus}). The teacher model (110.1M parameter BERT-base models as both encoders) for reranking achieves nDCG@10 of 42.7 and that for retrieval get nDCG@10 44.2.  The table shows performance (in nDCG@10) of the symmetric DE student model (67.5M or 11.3M transformer as both encoders), and asymmetric DE student model (67.5M or 11.3M transformer as query encoder and document embeddings inherited from the teacher).
    }
    \label{tab:msmarco_dede_ndcg}
    % \vskip 0.05in
    \vspace{1mm}
    \hspace{-2mm}
    \scalebox{0.85}{
    \begin{tabular}{@{}lccccc@{}}
    \toprule
    \multirow{2}{*}{\textbf{Method}} & \multicolumn{2}{c}{\textbf{Re-ranking}} && \multicolumn{2}{c}{\textbf{Retrieval}} \\
    \cmidrule{2-3} \cmidrule{5-6}
    & \textbf{67.5M} & \textbf{11.3M} && \textbf{67.5M} & \textbf{11.3M} \\
    \midrule
    Train student directly       & 32.2 & 29.7  && 27.2 & 22.5 \\
    \; + Distill from teacher     & 40.2 & 35.8 && 41.3  & 34.1  \\
    \; + Inherit doc embeddings   & 41.0 & 37.7 && 42.2 & 36.2  \\
    \; + Query embedding matching & {42.0} & \best{40.8} && \best{43.8} & \best{41.9}  \\
    \; + Query generation        &  {42.0} & 40.1 && \best{43.8} & 41.2 \\
    \midrule 
    Train student using only \\
    embedding matching and \\ 
    inherit doc embeddings & \best{42.3} & 39.3 && {43.3} & 37.6  \\
    \; + Query generation   & \best{42.3} & 39.9 && {43.4} & 39.2 \\
    \bottomrule
    \end{tabular}}
    \vspace{-2mm}
\end{table} 

\newpage
\subsection{Additional results on BEIR benchmark}
\label{appen:beir}

See Table~\ref{tab:beir-retrieval-ndcg} (NDCG@10) and Table~\ref{tab:beir-retrieval-recall} (Recall@100) for BEIR benchmark results. All numbers are from BEIR benchmark paper~\citep{thakur2021beir}. As common practice, non-public benchmark sets\footnote{https://github.com/beir-cellar/beir}, \{BioASQ, Signal-1M(RT), TREC-NEWS, Robust04\}, are removed from the table. Following the original BEIR paper~\citep{thakur2021beir} (Table~9 and Appendix~G from the original paper), we utilized Capped Recall@100 for TREC-COVID dataset. 

\newpage

\begin{table*}[t!]
    \caption{In-domain and zero-shot retrieval performance on BEIR benchmark~\citep{thakur2021beir}, as measured by \textbf{nDCG@10}. All the baseline number in the table are taken from \cite{thakur2021beir}. We exclude (in-domain) MSMARCO from average computation as common practice.}
    \label{tab:beir-retrieval-ndcg}
    \vskip 0.1in
    \small
    \resizebox{\textwidth}{!}{\begin{tabular}{@{}l  c  c c c  c c c c >{\centering\arraybackslash}m{2cm}  >{\centering\arraybackslash}m{2cm} @{}} % {l | c | c c c | c c c c c | c }
        \toprule
        \textbf{Model ($\rightarrow$)} &
        \multicolumn{1}{c}{Lexical} &
        \multicolumn{3}{c}{Sparse} &
        \multicolumn{6}{c}{Dense} \\ 
        \cmidrule(r){1-1}
        \cmidrule(lr){2-2}
        \cmidrule(lr){3-5}
        \cmidrule(l){6-11}
        \textbf{Dataset ($\downarrow$)} &
        \multicolumn{1}{c}{\textbf{BM25}} &
        \multicolumn{1}{c}{\textbf{DeepCT}} &
        \multicolumn{1}{c}{\textbf{SPARTA}} &
        \multicolumn{1}{c}{\textbf{docT5query}} &
        \multicolumn{1}{c}{\textbf{DPR}} &
        \multicolumn{1}{c}{\textbf{ANCE}} &
        \multicolumn{1}{c}{\textbf{TAS-B}} &
        \multicolumn{1}{c}{\textbf{GenQ}} &
        {\centering\textbf{SentenceBERT} \par \emph{(our teacher)}} &
        {\centering\textbf{\embd} \par \emph{(ours)}}\\
        \midrule
    MS MARCO & 22.8 & 29.6$^\ddagger$ & 35.1$^\ddagger$ & 33.8$^\ddagger$ & 17.7 & 38.8$^\ddagger$ & 40.8$^\ddagger$ & {40.8}$^\ddagger$ & 47.1$^\ddagger$ & {46.6}$^\ddagger$ \\ \midrule 
    TREC-COVID & 65.6 & 40.6 & 53.8 & {71.3} & 33.2 & 65.4 & 48.1 & 61.9 & 75.4 & {72.3}  \\
    NFCorpus & {32.5} & 28.3 & 30.1 & {32.8} & 18.9 & 23.7 & 31.9 & 31.9 & 31.0 & 30.7  \\
    NQ & 32.9 & 18.8 & 39.8 & 39.9 & {47.4}$^\ddagger$ & 44.6 & 46.3 & 35.8 & 51.5 & {50.8}  \\
    HotpotQA & {60.3} & 50.3 & 49.2 & 58.0 & 39.1 & 45.6 & {58.4} & 53.4 & 58.0 & 56.0  \\
    FiQA-2018 & 23.6 & 19.1 & 19.8 & 29.1 & 11.2 & {29.5} & 30.0 & {30.8} & 31.8 & {29.5}  \\
    ArguAna & 31.5 & 30.9 & 27.9 & 34.9 & 17.5 & 41.5 & {42.9} & {49.3} & 38.5 & 34.9  \\
    Touch\'e-2020 & {36.7} & 15.6 & 17.5 & {34.7} & 13.1 & 24.0 & 16.2 & 18.2 & 22.9 & 24.7  \\
    CQADupStack & 29.9 & 26.8 & 25.7 & {32.5} & 15.3 & 29.6 & 31.4 & {34.7} & 33.5 & 30.6  \\
    Quora & 78.9 & 69.1 & 63.0 & 80.2 & 24.8 & {85.2} & {83.5} & 83.0 & 84.2 & 81.4  \\
    DBPedia & 31.3 & 17.7 & 31.4 & 33.1 & 26.3 & 28.1 & {38.4} & 32.8 & 37.7 & {35.9}  \\
    SCIDOCS & {15.8} & 12.4 & 12.6 & {16.2} & 07.7 & 12.2 & 14.9 & 14.3 & 14.8 & 14.4  \\
    FEVER & {75.3} & 35.3 & 59.6 & 71.4 & 56.2 & 66.9 & 70.0 & 66.9 & 76.7 & {76.9}  \\
    Climate-FEVER & 21.3 & 06.6 & 08.2 & 20.1 & 14.8 & 19.8 & {22.8} & 17.5 & 23.5 & {22.5}  \\
    SciFact & {66.5} & 63.0 & 58.2 & {67.5} & 31.8 & 50.7 & 64.3 & 64.4 & 59.8 & 55.5  \\  \midrule
    
    AVG (w/o MSMARCO) & 43.0 & 31.0 & 35.5 & {44.4} & 25.5 & 40.5 & 42.8 & 42.5 & 45.7 & {44.0} \\ 
    \bottomrule
    \end{tabular}}
\end{table*}

\begin{table*}[t!]
    \caption{In-domain and zero-shot retrieval performance on BEIR benchmark~\citep{thakur2021beir}, as measured by \textbf{Recall@100}. All the baseline number in the table are taken from \cite{thakur2021beir}. $\ddagger$ indicates in-domain retrieval performance. $\ast$ indicates capped recall following original benchmark setup. We exclude (in-domain) MSMARCO from average computation as common practice.}
    \label{tab:beir-retrieval-recall}
    \vskip 0.1in
    \resizebox{\textwidth}{!}{\begin{tabular}{@{}l  c  c c c  c c c c >{\centering\arraybackslash}m{2cm}  >{\centering\arraybackslash}m{2cm} @{}}
        \toprule
        \textbf{Model ($\rightarrow$)} &
        \multicolumn{1}{c}{Lexical} &
        \multicolumn{3}{c}{Sparse} &
        \multicolumn{6}{c}{Dense} \\
        \cmidrule(r){1-1}
        \cmidrule(lr){2-2}
        \cmidrule(lr){3-5}
        \cmidrule(l){6-11}
        \textbf{Dataset ($\downarrow$)} &
        \multicolumn{1}{c}{\textbf{BM25}} &
        \multicolumn{1}{c}{\textbf{DeepCT}} &
        \multicolumn{1}{c}{\textbf{SPARTA}} &
        \multicolumn{1}{c}{\textbf{docT5query}} &
        \multicolumn{1}{c}{\textbf{DPR}} &
        \multicolumn{1}{c}{\textbf{ANCE}} &
        \multicolumn{1}{c}{\textbf{TAS-B}} &
        \multicolumn{1}{c}{\textbf{GenQ}} &
        {\centering\textbf{SentenceBERT} \par \hspace{1mm} \emph{(our teacher)}} &
        {\centering\textbf{\embd} \par \emph{(ours)}}\\
        \midrule
    MS MARCO      & 65.8 & 75.2$^\ddagger$ & 79.3$^\ddagger$ & 81.9$^\ddagger$ & 55.2 & 85.2$^\ddagger$ & {88.4}$^\ddagger$ & {88.4}$^\ddagger$ & 91.7$^\ddagger$ & {90.6}$^\ddagger$ \\ \midrule 
    TREC-COVID     & {49.8}$^*$ & 34.7$^*$ & 40.9$^*$ & {54.1}$^*$ & 21.2$^*$ & 45.7$^*$ & 38.7$^*$ & 45.6$^*$ & 54.1$^*$ & 48.8$^*$  \\
    NFCorpus       & 25.0 & 23.5 & 24.3 & 25.3 & 20.8 & 23.2 & {28.0} & {28.0} & 27.7 & {26.7}  \\
    NQ             & 76.0 & 63.6 & 78.7 & 83.2 & 88.0$^\ddagger$ & 83.6 & {90.3} & 86.2 & 91.1 & {89.9}  \\
    HotpotQA       & {74.0} & {73.1} & 65.1 & 70.9 & 59.1 & 57.8 & 72.8 & 67.3 & 69.7 & 68.3  \\
    FiQA-2018      & 53.9 & 48.9 & 44.6 & 59.8 & 34.2 & 58.1 & 59.3 & {61.8} & 62.0 & {60.1}  \\
    ArguAna        & 94.2 & 93.2 & 89.3 & {97.2} & 75.1 & 93.7 & 94.2 & {97.8} & 89.2 & 87.8  \\
    Touch\'e-2020  & {53.8} & 40.6 & 38.1 & {55.7} & 30.1 & 45.8 & 43.1 & 45.1 & 45.3 & 45.5  \\
    CQADupStack    & 60.6 & 54.5 & 52.1 & {63.8} & 40.3 & 57.9 & 62.2 & {65.4} & 63.9 & 61.3  \\
    Quora          & 97.3 & 95.4 & 89.6 & 98.2 & 47.0 & {98.7} & 98.6 & {98.8} & 98.5 & 98.1  \\
    DBPedia        & 39.8 & 37.2 & 41.1 & 36.5 & 34.9 & 31.9 & {49.9} & {43.1} & 46.0 & 42.6  \\
    SCIDOCS        & {35.6} & 31.4 & 29.7 & {36.0} & 21.9 & 26.9 & 33.5 & 33.2 & 32.5 & 31.5  \\
    FEVER          & 93.1 & 73.5 & 84.3 & 91.6 & 84.0 & 90.0 & {93.7} & 92.8 & 93.9 & {93.8}  \\
    Climate-FEVER  & 43.6 & 23.2 & 22.7 & 42.7 & 39.0 & 44.5 & {53.4} & 45.0 & 49.3 & {47.6}  \\
    SciFact        & {90.8} & 89.3 & 86.3 & {91.4} & 72.7 & 81.6 & 89.1 & 89.3 & 88.9 & 87.2  \\ \midrule
    AVG (w/o MSMARCO) & 63.4 & 55.9 & 56.2 & {64.7} & 47.7 & 60.0 & {64.8} & 64.2 & 65.1 & 63.5  \\
    \bottomrule
    \end{tabular}}
\end{table*}

\clearpage
\subsection{Additional results with single-stage trained teachers}
\label{appendix:old-teacher-results}
Hereby we evaluate \embd~with a simple single-stage trained teachers instead of teachers trained in complex multi-stage frameworks, in order to test the generalizability of the method.

Similar to Table~\ref{tab:nq-dev}, we conducted an experiment on top of single-stage trained teacher based on RoBERTa-base instead of AR2~\citep{zhang2022ar2} in the main text.
We also changed the student to be based on DistilRoBERTa or RoBERTa-mini accordingly for simplicity to use same tokenizer.

Table ~\ref{tab:nq-dev-old-teacher} demonstrates that \embd~provides a significant boost of the performance on top of standard distillation techniques similar to what we observed in Table~\ref{tab:nq-dev}.

\begin{table}[h]
    \vspace{-1em}
    \caption{\textit{Full} recall performance of various student DE models on NQ dev set, including symmetric DE student model, and asymmetric DE student models. All students used the same \emph{in-house teacher} (124M parameter RoBERTa-base models as both encoders), with the full Recall@5 = 64.6, Recall@20 = 81.7, and Recall@100 = 91.5.
    }\label{tab:nq-dev-old-teacher}
    \centering
    \vspace{2mm}
    \scalebox{0.85}{
    \begin{tabular}{@{}lc@{\hspace{2.0mm}}c@{ }cc@{\hspace{-2mm} }c@{\hspace{2.0mm}}c@{ }c@{}}
    \toprule
    \multirow{2}{*}{\textbf{Method}} & \multicolumn{3}{c}{\textbf{6-Layer (82M)}} && \multicolumn{3}{c}{\textbf{4-Layer (16M)}}\\
    \cmidrule(r){2-4} \cmidrule{6-8}
    & \textbf{R@5} & \textbf{R@20} & \textbf{R@100} && \textbf{R@5} & \textbf{R@20} & \textbf{R@100} \\
    \midrule
    Train student directly           & 41.9 & 64.5 & 82.0 && 39.5 & 59.9 & 76.3 \\
    \; + Distill from teacher        & 48.3 & 67.2 & 80.9 && 44.9 & 61.1 & 74.8 \\
    \; + Inherit doc embeddings & 56.9 & 74.3 & 85.4 && 47.2 & 64.0 & 77.0 \\
    \; + Query embedding matching    & 61.8 & 78.7 & 89.0 && 56.7 & 74.6 & 85.9 \\
    \; + Query generation            & 61.7 & 79.4 & 89.6 && 57.1 & 75.2 & \best{86.7} \\
    \midrule
    Train student using only\\
    embedding matching and \\
    inherit doc embeddings & 63.7 & 80.3 & 90.3 &&  57.9 & 74.6 & 85.7 \\
    \; + Query generation           & \best{64.1} & \best{80.5} & \best{90.4}  && \best{58.9} & \best{76.0} & 86.6 \\
    \bottomrule
    \end{tabular}}
\end{table}

Furthermore, we also consider a in-house trained teacher (RoBERTa-base) for MSMARCO re-ranking task. Table~\ref{tab:msmarco-dede-old-teachder} demonstrates a similar pattern to Table~\ref{tab:msmarco_dede}, providing evidence of generalizability of \embd.

\begin{table}[h]
    \vspace{-1em}
    \caption{Reranking performance of various DE models on MSMARCO dev set. We utilize a RoBERTa-base in-house trained teacher achieving MRR@10 of 33.1 and nDCG@10 of 38.8 is used. The table shows performance of the symmetric DE student model and asymmetric DE student models.}
    \vspace{2mm}
    \centering
    \scalebox{0.85}{
    \begin{tabular}{@{}lccccc@{}}
    \toprule
    \multirow{2}{*}{\textbf{Method}} & \multicolumn{2}{c}{\textbf{MRR@10}} && \multicolumn{2}{c}{\textbf{nDCG@10}} \\
    \cmidrule{2-3} \cmidrule{5-6}
    & \textbf{82M} & \textbf{16M} && \textbf{82M} & \textbf{16M} \\
    \midrule
    Train student directly           & 29.7 & 26.3 && 35.2 & 31.4 \\
    \quad + Distill from teacher     & 31.6 & 28.4 && 37.2 & 33.5 \\
    \quad + Inherit doc embeddings   & 32.4 & 30.2 && 38.0 & 35.8 \\
    \quad + Query embedding matching & 32.8 & 31.9 && 38.6 & 37.6 \\
    \quad + Query generation         & 33.0 & \best{32.0} && 38.8 & \best{37.7} \\
    \midrule 
    Train student only using embedding \\
    matching and inherit doc embeddings & 32.7 & 31.8 && 38.5 & 37.5 \\
    \quad + Query generation            & \best{33.0} & 31.8 && \best{38.9} & 37.5 \\
    \bottomrule
    \end{tabular}}
    \label{tab:msmarco-dede-old-teachder}
    \vspace{-3mm}
\end{table} 

These result showcase that our method brings performance boost orthogonal to how teacher was trained, whether single-staged or multi-staged.

\clearpage
\section{Embedding analysis}\label{appen:exp-analysis}

\subsection{DE to DE distillation}
\label{appen:embed-analysis-de2de}

Traditional score matching-based distillation might not result in transfer of relative geometry from teacher to student. To assess this, we look at the discrepancy between the teacher and student query embeddings for all $q,q'$ pairs: $\|{\tt emb}^{\rm t}_{q} - {\tt emb}^{\rm t}_{q'}\| - \|{\tt emb}^{\rm s}_{q} - {\tt emb}^{\rm s}_{q'}\|$. Note that the analysis is based on NQ, and we focus on the teacher and student DE models based on BERT-base and DistilBERT, respectively.
As evident from Fig.~\ref{fig:de_distance}, embedding matching loss significantly reduces this discrepancy.

\subsection{CE to DE distillation}
\label{appen:embed-analysis-ce2de}

% \begin{figure}[bh]
% \vspace{5mm}
\begin{wrapfigure}{r}{0.41\textwidth}
    \vspace{-9mm}
    \centering
    \includegraphics[width=0.96\linewidth]{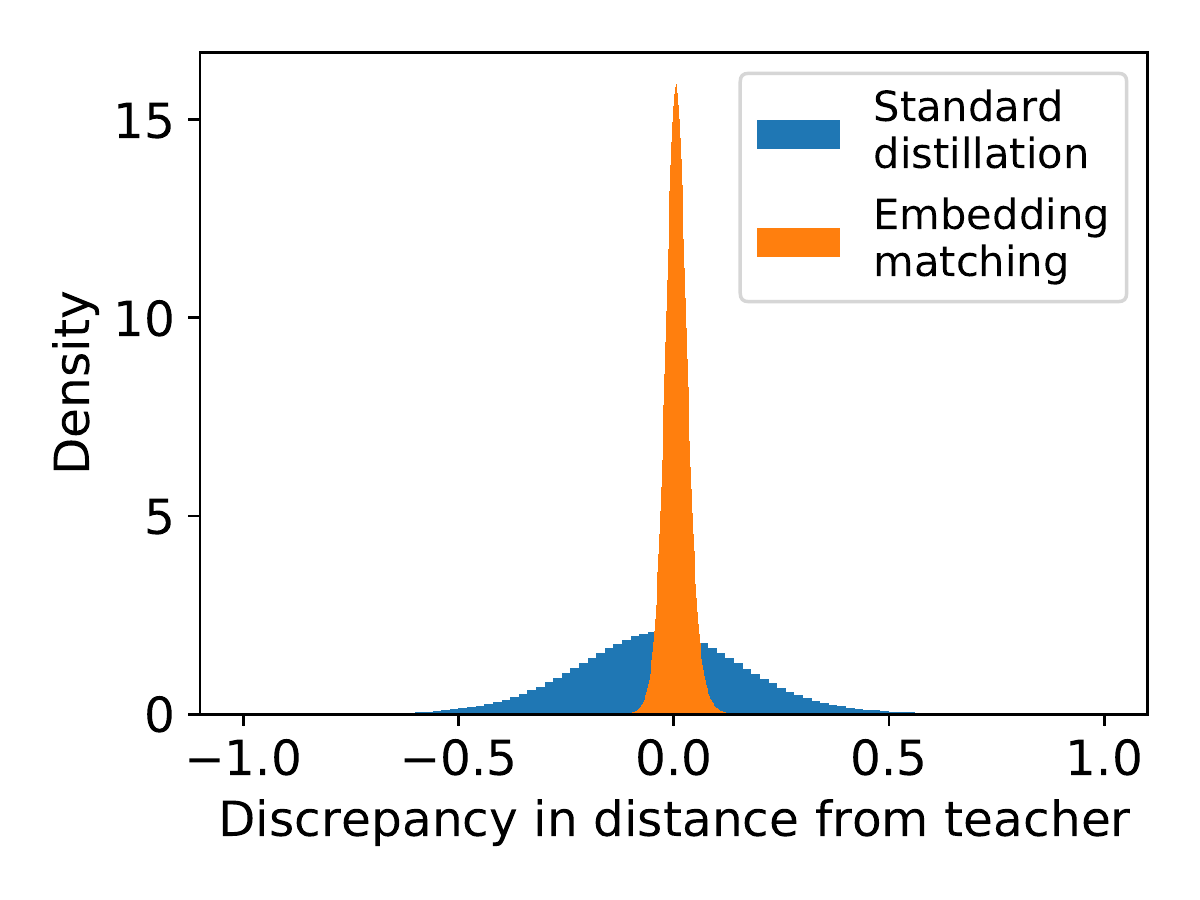}
    \vspace{-5mm}
    \caption{Histogram of teacher-student distance discrepancy in queries.}
    \label{fig:de_distance}
    \vspace{5mm}
\end{wrapfigure}
% \end{figure}

We qualitatively look at embeddings from CE model in Fig.~\ref{fig:ce-visualize}.
The embedding ${\tt emb}^{\rm t}_{q, d}$ from \texttt{[CLS]}-pooled CE model does not capture semantic similarity between query and document as it is solely trained to classify whether the query-document pair is relevant or not. In contrast, the (proxy) query embeddings ${\tt emb}^{\rm t}_{q \leftarrow (q,d)}$ from our Dual-pooled CE model with reconstruction loss do not degenerate and its embeddings groups same query whether conditioned on positive or negative document together. Furthermore, other related queries are closer than unrelated queries. Such informative embedding space would aid distillation to a DE model via embedding matching. 

\vspace{10mm}

\begin{figure}[bh]
    \centering
    \includegraphics[page=10,width=0.6\linewidth,trim=0 36mm 9cm 0,clip]{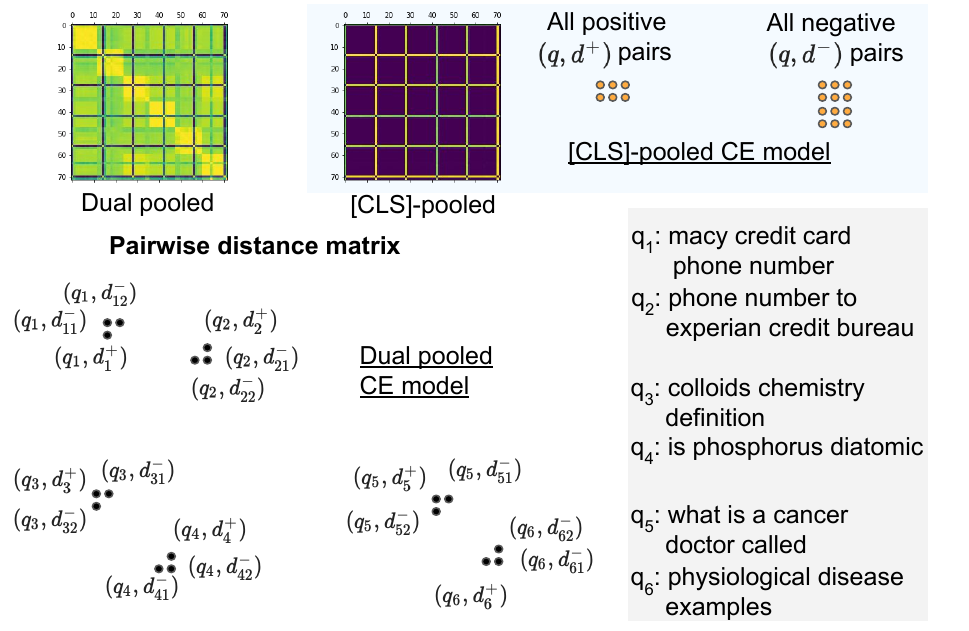}
    \vspace{-2mm}
    \caption{Illustration of geometry expressed by \texttt{[CLS]}-pooled CE and our Dual-pooled CE model on 6 queries from MSMARCO and 12 passages based on pairwise distance matrix across these 72 pairs. \texttt{[CLS]}-pooled CE embeddings degenerates as all positive and negative query-document pairs almost collapse to two points and fail to capture semantic information. In contrast, our Dual-pooled CE model leads to much richer representation that can express semantic information.}
    \label{fig:ce-visualize}
\end{figure}

%% file: arxiv_v2.bbl
\begin{thebibliography}{67}
\providecommand{\natexlab}[1]{#1}
\providecommand{\url}[1]{\texttt{#1}}
\expandafter\ifx\csname urlstyle\endcsname\relax
  \providecommand{\doi}[1]{doi: #1}\else
  \providecommand{\doi}{doi: \begingroup \urlstyle{rm}\Url}\fi

\bibitem[Aguilar et~al.(2020)Aguilar, Ling, Zhang, Yao, Fan, and
  Guo]{aguilar2020knowledge}
Gustavo Aguilar, Yuan Ling, Yu~Zhang, Benjamin Yao, Xing Fan, and Chenlei Guo.
\newblock Knowledge distillation from internal representations.
\newblock In \emph{Proceedings of the AAAI Conference on Artificial
  Intelligence}, volume~34, pages 7350--7357, 2020.

\bibitem[Alberti et~al.(2019)Alberti, Andor, Pitler, Devlin, and
  Collins]{Alberti:2019Synthetic}
Chris Alberti, Daniel Andor, Emily Pitler, Jacob Devlin, and Michael Collins.
\newblock Synthetic {QA} corpora generation with roundtrip consistency.
\newblock In \emph{Proceedings of the 57th Annual Meeting of the Association
  for Computational Linguistics}, pages 6168--6173, Florence, Italy, July 2019.
  Association for Computational Linguistics.
\newblock \doi{10.18653/v1/P19-1620}.
\newblock URL \url{https://aclanthology.org/P19-1620}.

\bibitem[Bengio and Senecal(2008)]{bengio2008sampled}
Yoshua Bengio and Jean-SÉbastien Senecal.
\newblock Adaptive importance sampling to accelerate training of a neural
  probabilistic language model.
\newblock \emph{IEEE Transactions on Neural Networks}, 19\penalty0
  (4):\penalty0 713--722, 2008.
\newblock \doi{10.1109/TNN.2007.912312}.

\bibitem[Bousquet et~al.(2004)Bousquet, Boucheron, and Lugosi]{Bousquet2004}
Olivier Bousquet, St{\'e}phane Boucheron, and G{\'a}bor Lugosi.
\newblock \emph{Introduction to Statistical Learning Theory}, pages 169--207.
\newblock Springer Berlin Heidelberg, Berlin, Heidelberg, 2004.
\newblock ISBN 978-3-540-28650-9.
\newblock \doi{10.1007/978-3-540-28650-9_8}.
\newblock URL \url{https://doi.org/10.1007/978-3-540-28650-9_8}.

\bibitem[Bucil\v{a} et~al.(2006)Bucil\v{a}, Caruana, and
  Niculescu-Mizil]{Bucilua:2006}
Cristian Bucil\v{a}, Rich Caruana, and Alexandru Niculescu-Mizil.
\newblock Model compression.
\newblock In \emph{Proceedings of the 12th ACM SIGKDD International Conference
  on Knowledge Discovery and Data Mining}, KDD '06, pages 535--541, New York,
  NY, USA, 2006. ACM.

\bibitem[Chen et~al.(2017)Chen, Fisch, Weston, and Bordes]{Chen:2017Reading}
Danqi Chen, Adam Fisch, Jason Weston, and Antoine Bordes.
\newblock Reading {W}ikipedia to answer open-domain questions.
\newblock In \emph{Proceedings of the 55th Annual Meeting of the Association
  for Computational Linguistics (Volume 1: Long Papers)}, pages 1870--1879,
  Vancouver, Canada, July 2017. Association for Computational Linguistics.
\newblock \doi{10.18653/v1/P17-1171}.
\newblock URL \url{https://aclanthology.org/P17-1171}.

\bibitem[Chen et~al.(2022)Chen, Mei, Zhang, Wang, Feng, and
  Chen]{chen2022knowledge}
Defang Chen, Jian-Ping Mei, Hailin Zhang, Can Wang, Yan Feng, and Chun Chen.
\newblock Knowledge distillation with the reused teacher classifier.
\newblock In \emph{Proceedings of the IEEE/CVF Conference on Computer Vision
  and Pattern Recognition}, pages 11933--11942, 2022.

\bibitem[Chen et~al.(2021)Chen, He, Hui, Sun, and Sun]{Chen:2021KD}
Xuanang Chen, Ben He, Kai Hui, Le~Sun, and Yingfei Sun.
\newblock Simplified tinybert: Knowledge distillation for document retrieval.
\newblock In Djoerd Hiemstra, Marie-Francine Moens, Josiane Mothe, Raffaele
  Perego, Martin Potthast, and Fabrizio Sebastiani, editors, \emph{Advances in
  Information Retrieval}, pages 241--248, Cham, 2021. Springer International
  Publishing.
\newblock ISBN 978-3-030-72240-1.

\bibitem[Dai and Callan(2019{\natexlab{a}})]{Dai:2019}
Zhuyun Dai and Jamie Callan.
\newblock Deeper text understanding for {IR} with contextual neural language
  modeling.
\newblock In Benjamin Piwowarski, Max Chevalier, {\'{E}}ric Gaussier, Yoelle
  Maarek, Jian{-}Yun Nie, and Falk Scholer, editors, \emph{Proceedings of the
  42nd International {ACM} {SIGIR} Conference on Research and Development in
  Information Retrieval, {SIGIR} 2019, Paris, France, July 21-25, 2019}, pages
  985--988. {ACM}, 2019{\natexlab{a}}.

\bibitem[Dai and Callan(2019{\natexlab{b}})]{Dai:2019DeepCT}
Zhuyun Dai and Jamie Callan.
\newblock Context-aware sentence/passage term importance estimation for first
  stage retrieval.
\newblock \emph{arXiv preprint arXiv:1910.10687}, 2019{\natexlab{b}}.

\bibitem[Devlin et~al.(2019)Devlin, Chang, Lee, and Toutanova]{Devlin:2019}
Jacob Devlin, Ming{-}Wei Chang, Kenton Lee, and Kristina Toutanova.
\newblock {BERT:} pre-training of deep bidirectional transformers for language
  understanding.
\newblock In Jill Burstein, Christy Doran, and Thamar Solorio, editors,
  \emph{Proceedings of the 2019 Conference of the North American Chapter of the
  Association for Computational Linguistics: Human Language Technologies,
  {NAACL-HLT} 2019, Minneapolis, MN, USA, June 2-7, 2019, Volume 1 (Long and
  Short Papers)}, pages 4171--4186. Association for Computational Linguistics,
  2019.

\bibitem[Guo et~al.(2020)Guo, Sun, Lindgren, Geng, Simcha, Chern, and
  Kumar]{Guo:2020ScaNN}
Ruiqi Guo, Philip Sun, Erik Lindgren, Quan Geng, David Simcha, Felix Chern, and
  Sanjiv Kumar.
\newblock Accelerating large-scale inference with anisotropic vector
  quantization.
\newblock In \emph{International Conference on Machine Learning}, 2020.
\newblock URL \url{https://arxiv.org/abs/1908.10396}.

\bibitem[Hinton et~al.(2015)Hinton, Vinyals, and Dean]{Hinton:2015}
Geoffrey Hinton, Oriol Vinyals, and Jeff Dean.
\newblock Distilling the knowledge in a neural network, 2015.

\bibitem[Hofst{\"{a}}tter et~al.(2020)Hofst{\"{a}}tter, Althammer,
  Schr{\"{o}}der, Sertkan, and Hanbury]{Hofstatter:2020}
Sebastian Hofst{\"{a}}tter, Sophia Althammer, Michael Schr{\"{o}}der, Mete
  Sertkan, and Allan Hanbury.
\newblock Improving efficient neural ranking models with cross-architecture
  knowledge distillation.
\newblock \emph{CoRR}, abs/2010.02666, 2020.
\newblock URL \url{https://arxiv.org/abs/2010.02666}.

\bibitem[Hofst{\"{a}}tter et~al.(2021)Hofst{\"{a}}tter, Lin, Yang, Lin, and
  Hanbury]{Hofstatter:2021SIGR}
Sebastian Hofst{\"{a}}tter, Sheng-Chieh Lin, Jheng-Hong Yang, Jimmy Lin, and
  Allan Hanbury.
\newblock Efficiently teaching an effective dense retriever with balanced topic
  aware sampling.
\newblock In \emph{Proceedings of the 44th International ACM SIGIR Conference
  on Research and Development in Information Retrieval}, SIGIR '21, page
  113–122, New York, NY, USA, 2021. Association for Computing Machinery.
\newblock ISBN 9781450380379.
\newblock \doi{10.1145/3404835.3462891}.
\newblock URL \url{https://doi.org/10.1145/3404835.3462891}.

\bibitem[Izacard and Grave(2021)]{Izacard:2021Distilling}
Gautier Izacard and Edouard Grave.
\newblock Distilling knowledge from reader to retriever for question answering.
\newblock In \emph{International Conference on Learning Representations}, 2021.
\newblock URL \url{https://openreview.net/forum?id=NTEz-6wysdb}.

\bibitem[Izacard et~al.(2021)Izacard, Caron, Hosseini, Riedel, Bojanowski,
  Joulin, and Grave]{Izacard:2021Contriever}
Gautier Izacard, Mathild Caron, Lucas Hosseini, Sebastian Riedel, Piotr
  Bojanowski, Armand Joulin, and Edouard Grave.
\newblock Unsupervised dense information retrieval with contrastive learning.
\newblock \emph{arXiv preprint arXiv:2112.09118}, 2021.

\bibitem[Jiao et~al.(2020)Jiao, Yin, Shang, Jiang, Chen, Li, Wang, and
  Liu]{Jiao:2020TinyBERT}
Xiaoqi Jiao, Yichun Yin, Lifeng Shang, Xin Jiang, Xiao Chen, Linlin Li, Fang
  Wang, and Qun Liu.
\newblock {T}iny{BERT}: Distilling {BERT} for natural language understanding.
\newblock In \emph{Findings of the Association for Computational Linguistics:
  EMNLP 2020}, pages 4163--4174, Online, November 2020. Association for
  Computational Linguistics.
\newblock \doi{10.18653/v1/2020.findings-emnlp.372}.
\newblock URL \url{https://aclanthology.org/2020.findings-emnlp.372}.

\bibitem[Johnson et~al.(2021)Johnson, Douze, and Jégou]{Johnson:2021FAISS}
Jeff Johnson, Matthijs Douze, and Hervé Jégou.
\newblock Billion-scale similarity search with gpus.
\newblock \emph{IEEE Transactions on Big Data}, 7\penalty0 (3):\penalty0
  535--547, 2021.
\newblock \doi{10.1109/TBDATA.2019.2921572}.

\bibitem[Karpukhin et~al.(2020{\natexlab{a}})Karpukhin, Oguz, Min, Lewis, Wu,
  Edunov, Chen, and Yih]{Karpukhin:2020}
Vladimir Karpukhin, Barlas Oguz, Sewon Min, Patrick Lewis, Ledell Wu, Sergey
  Edunov, Danqi Chen, and Wen-tau Yih.
\newblock Dense passage retrieval for open-domain question answering.
\newblock In \emph{Proceedings of the 2020 Conference on Empirical Methods in
  Natural Language Processing (EMNLP)}, pages 6769--6781, Online, November
  2020{\natexlab{a}}. Association for Computational Linguistics.

\bibitem[Karpukhin et~al.(2020{\natexlab{b}})Karpukhin, O{\u{g}}uz, Min, Lewis,
  Wu, Edunov, Chen, and Yih]{karpukhin2020dense}
Vladimir Karpukhin, Barlas O{\u{g}}uz, Sewon Min, Patrick Lewis, Ledell Wu,
  Sergey Edunov, Danqi Chen, and Wen-tau Yih.
\newblock Dense passage retrieval for open-domain question answering.
\newblock \emph{arXiv preprint arXiv:2004.04906}, 2020{\natexlab{b}}.

\bibitem[Khattab and Zaharia(2020)]{Khattab:2020}
Omar Khattab and Matei Zaharia.
\newblock \emph{ColBERT: Efficient and Effective Passage Search via
  Contextualized Late Interaction over BERT}, page 39–48.
\newblock Association for Computing Machinery, New York, NY, USA, 2020.
\newblock ISBN 9781450380164.

\bibitem[Kwiatkowski et~al.(2019{\natexlab{a}})Kwiatkowski, Palomaki, Redfield,
  Collins, Parikh, Alberti, Epstein, Polosukhin, Devlin, Lee, Toutanova, Jones,
  Kelcey, Chang, Dai, Uszkoreit, Le, and Petrov]{Kwiatkowski:2019NQ}
Tom Kwiatkowski, Jennimaria Palomaki, Olivia Redfield, Michael Collins, Ankur
  Parikh, Chris Alberti, Danielle Epstein, Illia Polosukhin, Jacob Devlin,
  Kenton Lee, Kristina Toutanova, Llion Jones, Matthew Kelcey, Ming-Wei Chang,
  Andrew~M. Dai, Jakob Uszkoreit, Quoc Le, and Slav Petrov.
\newblock Natural questions: A benchmark for question answering research.
\newblock \emph{Transactions of the Association for Computational Linguistics},
  7:\penalty0 452--466, 2019{\natexlab{a}}.
\newblock \doi{10.1162/tacl_a_00276}.
\newblock URL \url{https://aclanthology.org/Q19-1026}.

\bibitem[Kwiatkowski et~al.(2019{\natexlab{b}})Kwiatkowski, Palomaki, Redfield,
  Collins, Parikh, Alberti, Epstein, Polosukhin, Devlin, Lee,
  et~al.]{kwiatkowski2019natural}
Tom Kwiatkowski, Jennimaria Palomaki, Olivia Redfield, Michael Collins, Ankur
  Parikh, Chris Alberti, Danielle Epstein, Illia Polosukhin, Jacob Devlin,
  Kenton Lee, et~al.
\newblock Natural questions: a benchmark for question answering research.
\newblock \emph{Transactions of the Association for Computational Linguistics},
  7:\penalty0 453--466, 2019{\natexlab{b}}.

\bibitem[Ledoux and Talagrand(1991)]{ledoux1991probability}
Michel Ledoux and Michel Talagrand.
\newblock \emph{Probability in Banach spaces}.
\newblock Springer-Verlag, 1991.

\bibitem[Lee et~al.(2019)Lee, Chang, and Toutanova]{Lee:2019}
Kenton Lee, Ming{-}Wei Chang, and Kristina Toutanova.
\newblock Latent retrieval for weakly supervised open domain question
  answering.
\newblock In Anna Korhonen, David~R. Traum, and Llu{\'{\i}}s M{\`{a}}rquez,
  editors, \emph{Proceedings of the 57th Conference of the Association for
  Computational Linguistics, {ACL} 2019, Florence, Italy, July 28- August 2,
  2019, Volume 1: Long Papers}, pages 6086--6096. Association for Computational
  Linguistics, 2019.

\bibitem[Lewis et~al.(2020)Lewis, Liu, Goyal, Ghazvininejad, Mohamed, Levy,
  Stoyanov, and Zettlemoyer]{Lewis:2020BART}
Mike Lewis, Yinhan Liu, Naman Goyal, Marjan Ghazvininejad, Abdelrahman Mohamed,
  Omer Levy, Veselin Stoyanov, and Luke Zettlemoyer.
\newblock {BART}: Denoising sequence-to-sequence pre-training for natural
  language generation, translation, and comprehension.
\newblock In \emph{Proceedings of the 58th Annual Meeting of the Association
  for Computational Linguistics}, pages 7871--7880, Online, July 2020.
  Association for Computational Linguistics.
\newblock \doi{10.18653/v1/2020.acl-main.703}.
\newblock URL \url{https://aclanthology.org/2020.acl-main.703}.

\bibitem[Li et~al.(2020)Li, Yates, MacAvaney, He, and Sun]{Li:2020Parade}
Canjia Li, Andrew Yates, Sean MacAvaney, Ben He, and Yingfei Sun.
\newblock Parade: Passage representation aggregation for document reranking.
\newblock \emph{arXiv preprint arXiv:2008.09093}, 2020.

\bibitem[Lin et~al.(2021)Lin, Yang, and Lin]{Lin:2021Coupled}
Sheng-Chieh Lin, Jheng-Hong Yang, and Jimmy Lin.
\newblock In-batch negatives for knowledge distillation with tightly-coupled
  teachers for dense retrieval.
\newblock In \emph{Proceedings of the 6th Workshop on Representation Learning
  for NLP (RepL4NLP-2021)}, pages 163--173, Online, August 2021. Association
  for Computational Linguistics.
\newblock \doi{10.18653/v1/2021.repl4nlp-1.17}.
\newblock URL \url{https://aclanthology.org/2021.repl4nlp-1.17}.

\bibitem[Liu et~al.(2019)Liu, Ott, Goyal, Du, Joshi, Chen, Levy, Lewis,
  Zettlemoyer, and Stoyanov]{Liu:2019Roberta}
Yinhan Liu, Myle Ott, Naman Goyal, Jingfei Du, Mandar Joshi, Danqi Chen, Omer
  Levy, Mike Lewis, Luke Zettlemoyer, and Veselin Stoyanov.
\newblock Roberta: A robustly optimized bert pretraining approach.
\newblock \emph{arXiv preprint arXiv:1907.11692}, 2019.

\bibitem[Lu et~al.(2020)Lu, Jiao, and Zhang]{Lu:2020TwinBERT}
Wenhao Lu, Jian Jiao, and Ruofei Zhang.
\newblock Twinbert: Distilling knowledge to twin-structured compressed bert
  models for large-scale retrieval.
\newblock In \emph{Proceedings of the 29th ACM International Conference on
  Information \& Knowledge Management}, CIKM '20, page 2645–2652, New York,
  NY, USA, 2020. Association for Computing Machinery.
\newblock ISBN 9781450368599.
\newblock \doi{10.1145/3340531.3412747}.
\newblock URL \url{https://doi.org/10.1145/3340531.3412747}.

\bibitem[Luan et~al.(2021)Luan, Eisenstein, Toutanova, and
  Collins]{Luan:2021Sparse}
Yi~Luan, Jacob Eisenstein, Kristina Toutanova, and Michael Collins.
\newblock Sparse, dense, and attentional representations for text retrieval.
\newblock \emph{Transactions of the Association for Computational Linguistics},
  9:\penalty0 329--345, 2021.
\newblock \doi{10.1162/tacl_a_00369}.
\newblock URL \url{https://aclanthology.org/2021.tacl-1.20}.

\bibitem[Ma et~al.(2021)Ma, Korotkov, Yang, Hall, and
  McDonald]{Ma:2021ZeroShot}
Ji~Ma, Ivan Korotkov, Yinfei Yang, Keith Hall, and Ryan McDonald.
\newblock Zero-shot neural passage retrieval via domain-targeted synthetic
  question generation.
\newblock In \emph{Proceedings of the 16th Conference of the European Chapter
  of the Association for Computational Linguistics: Main Volume}, pages
  1075--1088, Online, April 2021. Association for Computational Linguistics.
\newblock \doi{10.18653/v1/2021.eacl-main.92}.
\newblock URL \url{https://aclanthology.org/2021.eacl-main.92}.

\bibitem[MacAvaney et~al.(2019{\natexlab{a}})MacAvaney, Yates, Cohan, and
  Goharian]{MacAvaney:CEDR}
Sean MacAvaney, Andrew Yates, Arman Cohan, and Nazli Goharian.
\newblock {CEDR}: Contextualized embeddings for document ranking.
\newblock In \emph{Proceedings of the 42nd International ACM SIGIR Conference
  on Research and Development in Information Retrieval}, SIGIR'19, page
  1101–1104, New York, NY, USA, 2019{\natexlab{a}}. Association for Computing
  Machinery.
\newblock ISBN 9781450361729.
\newblock \doi{10.1145/3331184.3331317}.
\newblock URL \url{https://doi.org/10.1145/3331184.3331317}.

\bibitem[MacAvaney et~al.(2019{\natexlab{b}})MacAvaney, Yates, Hui, and
  Frieder]{MacAvaney:2019}
Sean MacAvaney, Andrew Yates, Kai Hui, and Ophir Frieder.
\newblock Content-based weak supervision for ad-hoc re-ranking.
\newblock In \emph{Proceedings of the 42nd International ACM SIGIR Conference
  on Research and Development in Information Retrieval}, SIGIR'19, page
  993–996, New York, NY, USA, 2019{\natexlab{b}}. Association for Computing
  Machinery.
\newblock ISBN 9781450361729.
\newblock \doi{10.1145/3331184.3331316}.
\newblock URL \url{https://doi.org/10.1145/3331184.3331316}.

\bibitem[MacAvaney et~al.(2020)MacAvaney, Nardini, Perego, Tonellotto,
  Goharian, and Frieder]{MacAvaney:2020}
Sean MacAvaney, Franco~Maria Nardini, Raffaele Perego, Nicola Tonellotto, Nazli
  Goharian, and Ophir Frieder.
\newblock \emph{Efficient Document Re-Ranking for Transformers by Precomputing
  Term Representations}, page 49–58.
\newblock Association for Computing Machinery, New York, NY, USA, 2020.
\newblock ISBN 9781450380164.

\bibitem[Menon et~al.(2022)Menon, Jayasumana, Rawat, Kim, Reddi, and
  Kumar]{Menon:2022DE}
Aditya Menon, Sadeep Jayasumana, Ankit~Singh Rawat, Seungyeon Kim, Sashank
  Reddi, and Sanjiv Kumar.
\newblock In defense of dual-encoders for neural ranking.
\newblock In Kamalika Chaudhuri, Stefanie Jegelka, Le~Song, Csaba Szepesvari,
  Gang Niu, and Sivan Sabato, editors, \emph{Proceedings of the 39th
  International Conference on Machine Learning}, volume 162 of
  \emph{Proceedings of Machine Learning Research}, pages 15376--15400. PMLR,
  17--23 Jul 2022.
\newblock URL \url{https://proceedings.mlr.press/v162/menon22a.html}.

\bibitem[Mitra and Craswell(2018)]{Mitra:2018Now}
Bhaskar Mitra and Nick Craswell.
\newblock An introduction to neural information retrieval.
\newblock \emph{Foundations and Trends® in Information Retrieval}, 13\penalty0
  (1):\penalty0 1--126, 2018.
\newblock ISSN 1554-0669.
\newblock \doi{10.1561/1500000061}.
\newblock URL \url{http://dx.doi.org/10.1561/1500000061}.

\bibitem[Neelakantan et~al.(2022)Neelakantan, Xu, Puri, Radford, Han, Tworek,
  Yuan, Tezak, Kim, Hallacy, et~al.]{neelakantan2022text}
Arvind Neelakantan, Tao Xu, Raul Puri, Alec Radford, Jesse~Michael Han, Jerry
  Tworek, Qiming Yuan, Nikolas Tezak, Jong~Wook Kim, Chris Hallacy, et~al.
\newblock Text and code embeddings by contrastive pre-training.
\newblock \emph{arXiv preprint arXiv:2201.10005}, 2022.

\bibitem[Nguyen et~al.(2016)Nguyen, Rosenberg, Song, Gao, Tiwary, Majumder, and
  Deng]{Nguyen:2016}
Tri Nguyen, Mir Rosenberg, Xia Song, Jianfeng Gao, Saurabh Tiwary, Rangan
  Majumder, and Li~Deng.
\newblock {MS} {MARCO:} {A} human generated machine reading comprehension
  dataset.
\newblock In Tarek~Richard Besold, Antoine Bordes, Artur~S. d'Avila Garcez, and
  Greg Wayne, editors, \emph{Proceedings of the Workshop on Cognitive
  Computation: Integrating neural and symbolic approaches 2016}, volume 1773 of
  \emph{{CEUR} Workshop Proceedings}. CEUR-WS.org, 2016.

\bibitem[Ni et~al.(2022)Ni, Qu, Lu, Dai, Hernandez~Abrego, Ma, Zhao, Luan,
  Hall, Chang, and Yang]{ni-etal-2022-large}
Jianmo Ni, Chen Qu, Jing Lu, Zhuyun Dai, Gustavo Hernandez~Abrego, Ji~Ma,
  Vincent Zhao, Yi~Luan, Keith Hall, Ming-Wei Chang, and Yinfei Yang.
\newblock Large dual encoders are generalizable retrievers.
\newblock In \emph{Proceedings of the 2022 Conference on Empirical Methods in
  Natural Language Processing}, pages 9844--9855, Abu Dhabi, United Arab
  Emirates, December 2022. Association for Computational Linguistics.
\newblock URL \url{https://aclanthology.org/2022.emnlp-main.669}.

\bibitem[Nie et~al.(2020)Nie, Zhang, Geng, Ramamurthy, Song, and
  Jiang]{Nie:2020}
Ping Nie, Yuyu Zhang, Xiubo Geng, Arun Ramamurthy, Le~Song, and Daxin Jiang.
\newblock {DC-BERT:} decoupling question and document for efficient contextual
  encoding.
\newblock In Jimmy Huang, Yi~Chang, Xueqi Cheng, Jaap Kamps, Vanessa Murdock,
  Ji{-}Rong Wen, and Yiqun Liu, editors, \emph{Proceedings of the 43rd
  International {ACM} {SIGIR} conference on research and development in
  Information Retrieval, {SIGIR} 2020, Virtual Event, China, July 25-30, 2020},
  pages 1829--1832. {ACM}, 2020.
\newblock \doi{10.1145/3397271.3401271}.
\newblock URL \url{https://doi.org/10.1145/3397271.3401271}.

\bibitem[Nogueira and Cho(2019)]{Nogueira:2019}
Rodrigo Nogueira and Kyunghyun Cho.
\newblock Passage re-ranking with {BERT}.
\newblock \emph{CoRR}, abs/1901.04085, 2019.
\newblock URL \url{http://arxiv.org/abs/1901.04085}.

\bibitem[Nogueira et~al.(2019{\natexlab{a}})Nogueira, Lin, and
  Epistemic]{Nogueira:2019docTTTTTquery}
Rodrigo Nogueira, Jimmy Lin, and AI~Epistemic.
\newblock From doc2query to doctttttquery.
\newblock \emph{Online preprint}, 6, 2019{\natexlab{a}}.

\bibitem[Nogueira et~al.(2019{\natexlab{b}})Nogueira, Yang, Lin, and
  Cho]{Nogueira:2019docTOquery}
Rodrigo Nogueira, Wei Yang, Jimmy Lin, and Kyunghyun Cho.
\newblock Document expansion by query prediction.
\newblock \emph{arXiv preprint arXiv:1904.08375}, 2019{\natexlab{b}}.

\bibitem[Nogueira et~al.(2020)Nogueira, Jiang, Pradeep, and
  Lin]{Nogueira:2020Seq2Seq}
Rodrigo Nogueira, Zhiying Jiang, Ronak Pradeep, and Jimmy Lin.
\newblock Document ranking with a pretrained sequence-to-sequence model.
\newblock In \emph{Findings of the Association for Computational Linguistics:
  EMNLP 2020}, pages 708--718, Online, November 2020. Association for
  Computational Linguistics.
\newblock \doi{10.18653/v1/2020.findings-emnlp.63}.
\newblock URL \url{https://aclanthology.org/2020.findings-emnlp.63}.

\bibitem[O{\u{g}}uz et~al.(2021)O{\u{g}}uz, Lakhotia, Gupta, Lewis, Karpukhin,
  Piktus, Chen, Riedel, Yih, Gupta, et~al.]{ouguz2021domain}
Barlas O{\u{g}}uz, Kushal Lakhotia, Anchit Gupta, Patrick Lewis, Vladimir
  Karpukhin, Aleksandra Piktus, Xilun Chen, Sebastian Riedel, Wen-tau Yih,
  Sonal Gupta, et~al.
\newblock Domain-matched pre-training tasks for dense retrieval.
\newblock \emph{arXiv preprint arXiv:2107.13602}, 2021.

\bibitem[Qu et~al.(2021)Qu, Ding, Liu, Liu, Ren, Zhao, Dong, Wu, and
  Wang]{Qu:2021Rocket}
Yingqi Qu, Yuchen Ding, Jing Liu, Kai Liu, Ruiyang Ren, Wayne~Xin Zhao, Daxiang
  Dong, Hua Wu, and Haifeng Wang.
\newblock Rocket{QA}: An optimized training approach to dense passage retrieval
  for open-domain question answering.
\newblock In Kristina Toutanova, Anna Rumshisky, Luke Zettlemoyer, Dilek
  Hakkani{-}T{\"{u}}r, Iz~Beltagy, Steven Bethard, Ryan Cotterell, Tanmoy
  Chakraborty, and Yichao Zhou, editors, \emph{Proceedings of the 2021
  Conference of the North American Chapter of the Association for Computational
  Linguistics: Human Language Technologies, {NAACL-HLT} 2021, Online, June
  6-11, 2021}, pages 5835--5847. Association for Computational Linguistics,
  2021.

\bibitem[Raffel et~al.(2020)Raffel, Shazeer, Roberts, Lee, Narang, Matena,
  Zhou, Li, and Liu]{Raffel:2020T5}
Colin Raffel, Noam Shazeer, Adam Roberts, Katherine Lee, Sharan Narang, Michael
  Matena, Yanqi Zhou, Wei Li, and Peter~J. Liu.
\newblock Exploring the limits of transfer learning with a unified text-to-text
  transformer.
\newblock \emph{Journal of Machine Learning Research}, 21\penalty0
  (140):\penalty0 1--67, 2020.
\newblock URL \url{http://jmlr.org/papers/v21/20-074.html}.

\bibitem[Reimers et~al.(2019)Reimers, Gurevych, and
  Gurevych]{reimers-2019-sentence-bert}
Nils Reimers, Iryna Gurevych, and Iryna Gurevych.
\newblock Sentence-{BERT}: Sentence embeddings using siamese bert-networks.
\newblock In \emph{Proceedings of the 2019 Conference on Empirical Methods in
  Natural Language Processing}. Association for Computational Linguistics, 11
  2019.
\newblock URL \url{http://arxiv.org/abs/1908.10084}.

\bibitem[Ren et~al.(2021)Ren, Qu, Liu, Zhao, She, Wu, Wang, and
  Wen]{Ren:2021RocketV2}
Ruiyang Ren, Yingqi Qu, Jing Liu, Wayne~Xin Zhao, Qiaoqiao She, Hua Wu, Haifeng
  Wang, and Ji-Rong Wen.
\newblock Rocketqav2: A joint training method for dense passage retrieval and
  passage re-ranking.
\newblock In \emph{Proceedings of EMNLP}, 2021.

\bibitem[Romero et~al.(2014)Romero, Ballas, Kahou, Chassang, Gatta, and
  Bengio]{romero2014fitnets}
Adriana Romero, Nicolas Ballas, Samira~Ebrahimi Kahou, Antoine Chassang, Carlo
  Gatta, and Yoshua Bengio.
\newblock Fitnets: Hints for thin deep nets.
\newblock \emph{arXiv preprint arXiv:1412.6550}, 2014.

\bibitem[Sachan et~al.(2021)Sachan, Patwary, Shoeybi, Kant, Ping, Hamilton, and
  Catanzaro]{sachan-etal-2021-end}
Devendra Sachan, Mostofa Patwary, Mohammad Shoeybi, Neel Kant, Wei Ping,
  William~L. Hamilton, and Bryan Catanzaro.
\newblock End-to-end training of neural retrievers for open-domain question
  answering.
\newblock In \emph{Proceedings of the 59th Annual Meeting of the Association
  for Computational Linguistics and the 11th International Joint Conference on
  Natural Language Processing (Volume 1: Long Papers)}, pages 6648--6662,
  Online, August 2021. Association for Computational Linguistics.
\newblock \doi{10.18653/v1/2021.acl-long.519}.
\newblock URL \url{https://aclanthology.org/2021.acl-long.519}.

\bibitem[Sachan et~al.(2022)Sachan, Lewis, Joshi, Aghajanyan, Yih, Pineau, and
  Zettlemoyer]{Sachan:2022Improving}
Devendra~Singh Sachan, Mike Lewis, Mandar Joshi, Armen Aghajanyan, Wen-tau Yih,
  Joelle Pineau, and Luke Zettlemoyer.
\newblock Improving passage retrieval with zero-shot question generation.
\newblock \emph{arXiv preprint arXiv:2204.07496}, 2022.

\bibitem[Sanh et~al.(2019)Sanh, Debut, Chaumond, and Wolf]{Sanh:2019DistilBERT}
Victor Sanh, Lysandre Debut, Julien Chaumond, and Thomas Wolf.
\newblock Distilbert, a distilled version of bert: smaller, faster, cheaper and
  lighter.
\newblock \emph{arXiv preprint arXiv:1910.01108}, 2019.

\bibitem[Santhanam et~al.(2021)Santhanam, Khattab, Saad{-}Falcon, Potts, and
  Zaharia]{Santhanam:2021}
Keshav Santhanam, Omar Khattab, Jon Saad{-}Falcon, Christopher Potts, and Matei
  Zaharia.
\newblock Colbertv2: Effective and efficient retrieval via lightweight late
  interaction.
\newblock \emph{CoRR}, abs/2112.01488, 2021.

\bibitem[Thakur et~al.(2021)Thakur, Reimers, R{\"u}ckl{\'e}, Srivastava, and
  Gurevych]{thakur2021beir}
Nandan Thakur, Nils Reimers, Andreas R{\"u}ckl{\'e}, Abhishek Srivastava, and
  Iryna Gurevych.
\newblock {BEIR}: A heterogeneous benchmark for zero-shot evaluation of
  information retrieval models.
\newblock In \emph{Thirty-fifth Conference on Neural Information Processing
  Systems Datasets and Benchmarks Track (Round 2)}, 2021.
\newblock URL \url{https://openreview.net/forum?id=wCu6T5xFjeJ}.

\bibitem[Turc et~al.(2019)Turc, Chang, Lee, and Toutanova]{turc2019well}
Iulia Turc, Ming-Wei Chang, Kenton Lee, and Kristina Toutanova.
\newblock Well-read students learn better: On the importance of pre-training
  compact models.
\newblock \emph{arXiv preprint arXiv:1908.08962}, 2019.

\bibitem[Vaswani et~al.(2017)Vaswani, Shazeer, Parmar, Uszkoreit, Jones, Gomez,
  Kaiser, and Polosukhin]{Vaswani:2017}
Ashish Vaswani, Noam Shazeer, Niki Parmar, Jakob Uszkoreit, Llion Jones,
  Aidan~N. Gomez, \L{}ukasz Kaiser, and Illia Polosukhin.
\newblock Attention is all you need.
\newblock In \emph{Proceedings of the 31st International Conference on Neural
  Information Processing Systems}, NIPS'17, page 6000–6010, Red Hook, NY,
  USA, 2017. Curran Associates Inc.
\newblock ISBN 9781510860964.

\bibitem[Xiong et~al.(2021)Xiong, Xiong, Li, Tang, Liu, Bennett, Ahmed, and
  Overwijk]{Xiong:2021ANCE}
Lee Xiong, Chenyan Xiong, Ye~Li, Kwok-Fung Tang, Jialin Liu, Paul~N. Bennett,
  Junaid Ahmed, and Arnold Overwijk.
\newblock Approximate nearest neighbor negative contrastive learning for dense
  text retrieval.
\newblock In \emph{International Conference on Learning Representations}, 2021.
\newblock URL \url{https://openreview.net/forum?id=zeFrfgyZln}.

\bibitem[Yadav et~al.(2022)Yadav, Monath, Angell, Zaheer, and
  McCallum]{yadav2022efficient}
Nishant Yadav, Nicholas Monath, Rico Angell, Manzil Zaheer, and Andrew
  McCallum.
\newblock Efficient nearest neighbor search for cross-encoder models using
  matrix factorization.
\newblock In \emph{Proceedings of the 2022 Conference on Empirical Methods in
  Natural Language Processing}, pages 2171--2194, Abu Dhabi, United Arab
  Emirates, December 2022. Association for Computational Linguistics.
\newblock URL \url{https://aclanthology.org/2022.emnlp-main.140}.

\bibitem[Yilmaz et~al.(2019)Yilmaz, Yang, Zhang, and Lin]{Yilmaz:2019}
Zeynep~Akkalyoncu Yilmaz, Wei Yang, Haotian Zhang, and Jimmy Lin.
\newblock Cross-domain modeling of sentence-level evidence for document
  retrieval.
\newblock In \emph{Proceedings of the 2019 Conference on Empirical Methods in
  Natural Language Processing and the 9th International Joint Conference on
  Natural Language Processing (EMNLP-IJCNLP)}, pages 3490--3496, Hong Kong,
  China, November 2019. Association for Computational Linguistics.

\bibitem[Zhang et~al.(2022)Zhang, Gong, Shen, Lv, Duan, and Chen]{zhang2022ar2}
Hang Zhang, Yeyun Gong, Yelong Shen, Jiancheng Lv, Nan Duan, and Weizhu Chen.
\newblock Adversarial retriever-ranker for dense text retrieval.
\newblock In \emph{International Conference on Learning Representations}, 2022.
\newblock URL \url{https://openreview.net/forum?id=MR7XubKUFB}.

\bibitem[Zhang and Ma(2020)]{zhang2020improve}
Linfeng Zhang and Kaisheng Ma.
\newblock Improve object detection with feature-based knowledge distillation:
  Towards accurate and efficient detectors.
\newblock In \emph{International Conference on Learning Representations}, 2020.

\bibitem[Zhang et~al.(2019)Zhang, Yao, Sun, and Tay]{Zhang:2019Survey}
Shuai Zhang, Lina Yao, Aixin Sun, and Yi~Tay.
\newblock Deep learning based recommender system: A survey and new
  perspectives.
\newblock \emph{ACM Comput. Surv.}, 52\penalty0 (1), feb 2019.
\newblock ISSN 0360-0300.
\newblock \doi{10.1145/3285029}.
\newblock URL \url{https://doi.org/10.1145/3285029}.

\bibitem[Zhao et~al.(2021)Zhao, Xiong, Boyd-Graber, and
  Daum{\'e}~III]{Zhao:2021Distantly}
Chen Zhao, Chenyan Xiong, Jordan Boyd-Graber, and Hal Daum{\'e}~III.
\newblock Distantly-supervised dense retrieval enables open-domain question
  answering without evidence annotation.
\newblock In \emph{Proceedings of the 2021 Conference on Empirical Methods in
  Natural Language Processing}, pages 9612--9622, Online and Punta Cana,
  Dominican Republic, November 2021. Association for Computational Linguistics.
\newblock \doi{10.18653/v1/2021.emnlp-main.756}.
\newblock URL \url{https://aclanthology.org/2021.emnlp-main.756}.

\bibitem[Zheng et~al.(2020)Zheng, Hui, He, Han, Sun, and
  Yates]{Zheng:2020BERT-QE}
Zhi Zheng, Kai Hui, Ben He, Xianpei Han, Le~Sun, and Andrew Yates.
\newblock {BERT-QE}: {C}ontextualized {Q}uery {E}xpansion for {D}ocument
  {R}e-ranking.
\newblock In \emph{Findings of the Association for Computational Linguistics:
  EMNLP 2020}, pages 4718--4728, Online, November 2020. Association for
  Computational Linguistics.
\newblock \doi{10.18653/v1/2020.findings-emnlp.424}.
\newblock URL \url{https://aclanthology.org/2020.findings-emnlp.424}.

\end{thebibliography}
